\documentclass{article}

\usepackage[preprint]{neurips_2026}

\usepackage[utf8]{inputenc} 
\usepackage[T1]{fontenc}    
\usepackage{hyperref}       
\usepackage{url}            
\usepackage{booktabs}       
\usepackage{amsfonts}       
\usepackage{nicefrac}       
\usepackage{microtype}      
\usepackage{xcolor}         
\usepackage{amsmath}
\usepackage{graphicx}
\usepackage{subcaption}
\newcommand{\MyTitle}{Casper3D}
\usepackage{amsthm}
\usepackage{booktabs}
\usepackage{multirow}
\usepackage[table]{xcolor}
\usepackage{graphicx}
\usepackage{amssymb}
\theoremstyle{definition}
\newtheorem{definition}{Definition}
\newtheorem{proposition}{Proposition}
\newtheorem{lemma}{Lemma}
\usepackage{comment}
\usepackage{pifont}
\newcommand{\xmark}{\ding{55}}
\usepackage{booktabs}
\usepackage{multirow}
\usepackage{makecell}
\usepackage[table]{xcolor}
\usepackage{pifont}
\usepackage{pgfplots}
\usepackage{tikz}
\pgfplotsset{compat=1.18}

\title{Lightweight 3D Feature Pretraining by Bayesian Inversion of 2D Foundation Models}

\author{
Marwane Hariat$^{1}$ \quad Gianni Franchi$^{2}$ \quad David Filliat$^{2}$ \quad Antoine Manzanera$^{1}$\thanks{Corresponding author.}\\[2mm]
$^{1}$U2IS, ENSTA -- Institut Polytechnique de Paris, Palaiseau, France\\
$^{2}$Pôle Recherche, Agence Ministérielle pour l'IA de Défense, Palaiseau, France\\
\texttt{\{marwane.hariat, antoine.manzanera, gianni.franchi\}@ensta.fr}\\
\texttt{\{david.filliat\}@polytechnique.edu}
}

\begin{document}

\maketitle

\begin{abstract}
We present Casper3D, a lightweight probabilistic framework for converting noisy multi-view 2D foundation-model embeddings into a latent 3D semantic representation. We model view-level semantic features as noisy observations of an underlying 3D semantic state and infer this state with a set-based variational model that incorporates relative pose during multi-view reasoning. Casper3D is trained by predicting held-out semantic observations from novel viewpoints, while remaining aligned with visual and text semantic spaces for open-vocabulary 3D understanding. The framework is backbone-agnostic and applies to both language-aligned and self-supervised embeddings. Experiments show that Casper3D produces more stable 3D semantics than simple multi-view pooling, especially in ambiguous and noisy settings.
\end{abstract}
    
\section{Introduction}

A central goal of perception is to recover stable properties of the physical world from partial and view-dependent measurements. Images are such measurements: they do not directly reveal the underlying three-dimensional scene, but rather record its projection under a particular camera pose, illumination, occlusion pattern, and surrounding context. Modern pretrained vision and vision-language models, such as DINO~\cite{caron2021emerging,oquab2023dinov2,simeoni2025dinov3}, CLIP~\cite{radford2021learning,zhou2022extract,luddecke2022image}, and masked autoencoders~\cite{he2022masked}, transform these image measurements into powerful semantic feature spaces. These models have become widely used as backbones for 2D tasks including classification, segmentation, detection, and depth estimation~\cite{amir2021deep,hamilton2022unsupervised,gu2021open}. Yet their representations remain fundamentally attached to image coordinates.

In contrast, the physical world is organized in three dimensions. A point on a table, a chair, or a wall exists independently of the camera from which it is observed. Its semantic identity is not created by a particular view, even though each view may reveal 
a partial and noisy signature of it. This distinction is crucial for 3D scene understanding: we would like to attach feature embeddings not to pixels, but to physical points
\(
\{p_i \in \mathbb{R}^3\}_{i=1}^{N}.
\)
For each point \(p_i\), the desired embedding should represent an intrinsic property of that point in the scene: a view-invariant semantic state that persists across observations.

This motivates the following physical interpretation. Given a 3D point \(p_i\), there exists an unobserved latent representation
\(
z_i \in \mathbb{R}^d,
\)
which encodes the semantic content of the physical point. A 2D foundation model does not observe \(z_i\) directly. Instead, when the point is seen from a camera view \(j\), the model produces a feature observation
\(
x_{ij} \in \mathbb{R}^d
\)
that depends both on the latent state \(z_i\) and on the viewing condition \(T_j\). The image feature is therefore a measurement, not the object itself. It is affected by projection, scale, viewpoint, occlusion, local image context, and the inductive biases of the pretrained model.

From this perspective, lifting 2D features into 3D is not merely a geometric alignment problem. It is a Bayesian inverse problem: recover the hidden, view-independent semantic state of a physical point from multiple noisy, view-dependent measurements. Formally, given the set of observations
\(
\mathcal{C}_i = \{(x_{ij}, T_j)\}_{j=1}^{M_i},
\)
we seek to infer the latent state that explains them. The desired 3D representation is not a single observed 2D feature, nor an arbitrary average of observed features, but the latent cause underlying all of them.

Most existing approaches adopt a lift-and-aggregate strategy~\cite{zhang2022pointclip,peng2023openscene,song2025mv}. Using camera intrinsics, poses, and depth, each 3D point is projected into the images where it is visible. Features from a pretrained 2D model are sampled at the corresponding image locations and fused by average pooling, weighted averaging, voting, or similar deterministic rules:
\(
e_i = \sum_{j=1}^{M_i} w_{ij} x_{ij}.
\)
This approach is simple and often effective, but it implicitly assumes that the 3D feature can be estimated as a weighted mean of its 2D measurements.


In this paper, we propose \textbf{\MyTitle} (\textbf{Canonical Semantic Probabilistic Embeddings for 3D}) a probabilistic framework for lifting pretrained 2D foundation-model features into 3D. \MyTitle\ treats 3D feature lifting as a Bayesian inverse problem. For each 3D point \(p_i\), we assume a latent semantic variable \(z_i\), and model each 2D feature observation as noisy observations generated by a view-dependent measurement process:
\(
x_{ij} \sim p_\theta(x \mid z_i, T_j).
\)
The main objective is then to recover the posterior distribution
\(
p(z_i \mid \mathcal{C}_i),
\)
which represents what can be inferred about the intrinsic 3D semantic state from all available views. This posterior formulation naturally distinguishes between the physical representation of the point and the observations induced by particular cameras. It also provides a principled way to represent uncertainty: points seen from many consistent views should yield concentrated posteriors, whereas points seen from few or conflicting views should remain uncertain.

Since exact posterior is intractable, \MyTitle\ introduces a lightweight view-conditioned variational model. The encoder receives an unordered set of 2D feature measurements and viewing conditions,
\(
\mathcal{C}_i = \{(x_{ij}, T_j)\}_{j=1}^{M_i},
\)
and produces an amortized posterior approximation
\(
q_\phi(z_i \mid \mathcal{C}_i).
\)
The decoder is conditioned on both the inferred latent state and a target viewing condition, and is trained to predict held-out 2D feature observations. In this way, the model is forced to separate the intrinsic component of the point representation, encoded in \(z_i\), from the extrinsic component induced by viewpoint, encoded through \(T_j\). We use held-out feature prediction as a self-supervised objective for learning a view-stable latent 3D semantic representation.


Once inferred, the resulting 3D feature field can be used for multiple downstream tasks. We focus on open-vocabulary 3D semantic segmentation as a primary application.

\paragraph{Contributions.} To summarize, our contributions are as follows:
\begin{itemize}
    \item We formulate 2D-to-3D foundation-model feature lifting as a Bayesian inverse problem, where each 3D point has an unobserved view-invariant semantic state and each 2D feature is a noisy, view-dependent observation of that state.

    \item We introduce \textbf{\MyTitle}, a probabilistic latent-variable framework for learning 3D feature fields from multi-view 2D pretrained features. \MyTitle\ uses a lightweight view-conditioned variational model to infer posterior distributions over latent 3D point features from unordered, variable-size sets of 2D observations, and is trained through held-out view prediction to separate view-stable semantic content from view-dependent measurement effects.

    \item We demonstrate \MyTitle\ on open-vocabulary 3D semantic segmentation, showing improved robustness and accuracy over aggregation-based feature lifting methods.
\end{itemize}

\section{Related Work}
\label{sec:related_work}

\textbf{Multi-view 2D feature lifting.}
A common strategy for transferring 2D semantics to 3D is to associate image observations with geometry and fuse them across views. SemanticFusion~\cite{mccormac2017semanticfusion} follows this idea using SLAM~\cite{whelan2016elasticfusion}, but remains closed-set. Recent open-vocabulary methods replace supervised predictors with pretrained 2D models: ConceptFusion~\cite{jatavallabhula2023conceptfusion} lifts SAM-based~\cite{kirillov2023segment} and global image features into 3D; FeatureRealisticFusion~\cite{mazur2023feature} learns a neural field with an iMAP-like backend~\cite{sucar2021imap}; OpenMask3D~\cite{takmaz2023openmask3d}, OpenIns3D~\cite{huang2024openins3d}, and OV3D~\cite{zhou2025ov3d} lift SAM-guided masks~\cite{kirillov2023segment} to 3D instances. Multi-view recognition methods such as PointCLIP~\cite{zhang2022pointclip} and MV-CLIP~\cite{song2025mv} aggregate CLIP features or logits across rendered views. For point-level segmentation, OpenScene~\cite{peng2023openscene} and CLIP-FO3D~\cite{zhang2023clip} lift and distill CLIP-aligned features, while Semantic Abstraction~\cite{ha2022semantic}, later text-supervised methods~\cite{jiang2024open,zhang2025pgov3d}, RegionPLC~\cite{yang2024regionplc}, PLA/CLIP2Scene-style methods~\cite{ding2023pla,chen2023clip2scene}, and DITR~\cite{knaebel2025dino} use 2D foundation-model supervision through text, region, or feature correspondences. These works show the value of 2D foundation models for 3D perception, but typically obtain 3D features by selecting, averaging, distilling, or fusing view-dependent observations. We instead infer a posterior over an unobserved latent 3D semantic state.

\textbf{Bayesian inverse problems.}
Our formulation follows the spirit of Bayesian inverse problems, where an unknown latent quantity is inferred from noisy measurements rather than directly aggregated. For example, dark-matter mass mapping infers an unobserved matter field from observational tracers~\cite{royo2026mapping}. In our case, the desired 3D semantic feature is also unobserved and only measured through view-dependent 2D foundation-model features. Unlike physical inverse problems, where the latent field and measurement process may be physically defined, and recent methods use diffusion models~\cite{ho2020denoising} with Langevin-based or improved posterior samplers~\cite{zhang2025improving,wang2026sample}, our latent state is an abstract semantic variable with no ground-truth target. We therefore learn a lightweight amortized posterior \(q_\phi(z_i \mid \mathcal{C}_i)\).

\textbf{Set-based variational models.}
Our encoder must handle unordered sets with a variable number of views. Deep Sets~\cite{zaheer2017deep}, Set Transformers~\cite{lee2019set}, and SetVAE~\cite{kim2021setvae} provide tools for permutation-invariant and variational set modeling. Unlike standard set generation, and unlike order-dependent ViT positional encodings~\cite{dosovitskiy2020image}, our viewing condition is part of the physical measurement process. We therefore combine set-based posterior inference with view-conditioned decoding.

\textbf{Held-out feature prediction.}
Predicting held-out features from target views is related to view-conditioned novel-view synthesis~\cite{liu2023zero}. However, we do not render RGB images or geometry; we use held-out foundation-feature prediction as a self-supervised signal for learning a view-stable latent 3D representation.
\section{Method}
\label{sec:method}
\setlength{\abovedisplayskip}{4pt}
\setlength{\belowdisplayskip}{4pt}
\setlength{\abovedisplayshortskip}{3pt}
\setlength{\belowdisplayshortskip}{3pt}
\setlength{\textfloatsep}{8pt}
\setlength{\floatsep}{6pt}
\setlength{\intextsep}{6pt}

\begin{figure}[t]
  \centering
  \includegraphics[width=\linewidth, height=6.5cm]{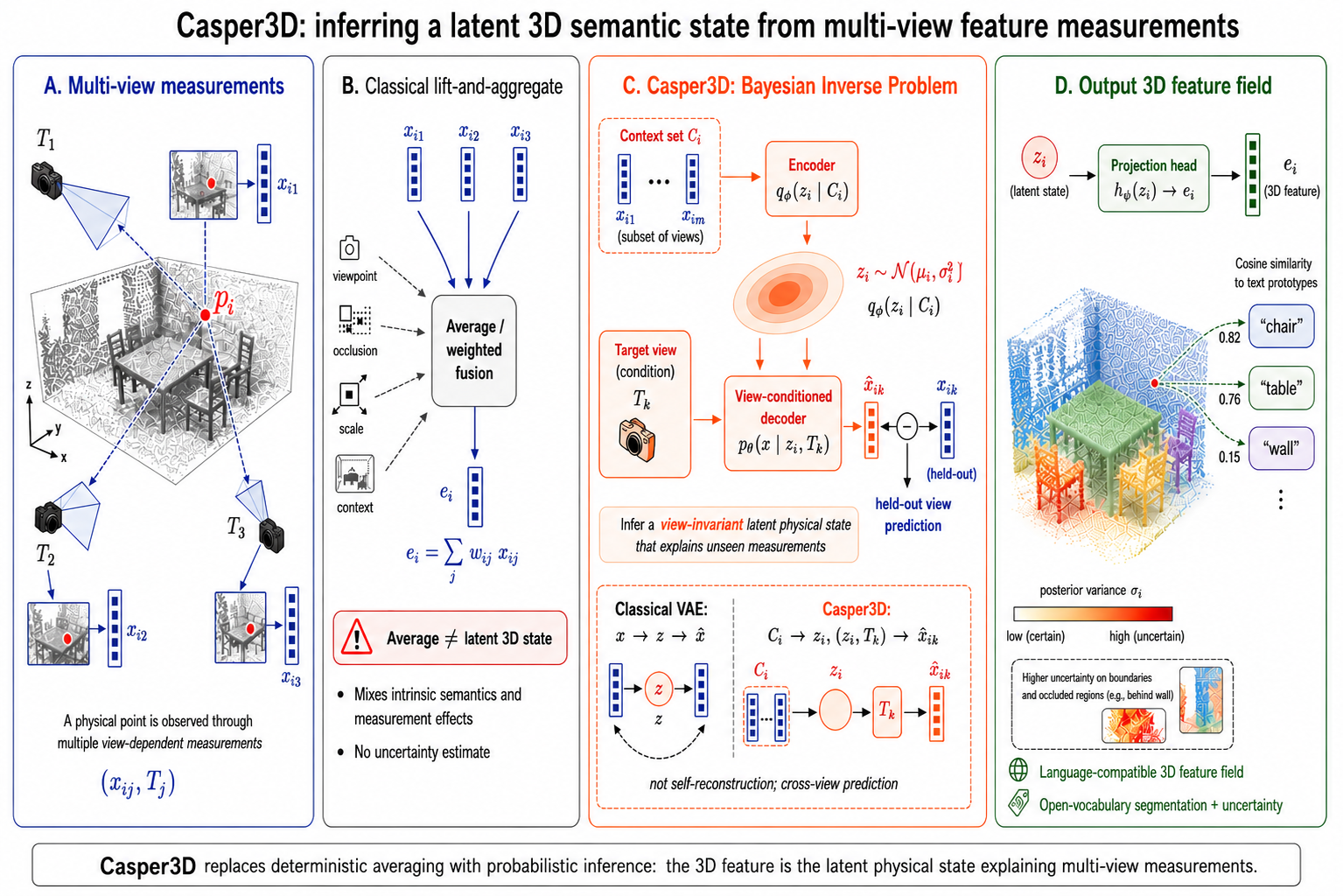}
  \caption{\textbf{Overview of \MyTitle}}
  \label{fig:full_pipeline}
\end{figure}

\textbf{\MyTitle} (\textbf{Canonical Semantic Probabilistic Embeddings for 3D}) is a probabilistic method for lifting pretrained 2D foundation-model features into 3D. Its central idea is that a 3D point has an underlying semantic state, but this state is never directly observed. We only observe its 2D measurements through different cameras.

This creates an important difficulty: there is no ground-truth latent 3D feature to supervise. We do not know what the ``correct'' internal 3D representation of CLIP, DINO, or any other vision-language model should be. We only observe view-dependent 2D features sampled from images. Therefore, learning a 3D feature field cannot be posed as standard supervised regression. Instead, we formulate it as probabilistic inference: the model must infer a distribution over possible latent 3D states that are compatible with the available measurements.

This is the main departure from deterministic feature lifting. Existing methods often collapse all observed 2D features into one embedding by averaging or weighted fusion. \MyTitle\ instead treats these observations as noisy physical measurements and asks: \emph{which latent 3D state could have generated them?}

\subsection{A Latent Physical State, Not an Averaged Feature}



As mentioned previously, \MyTitle\ therefore introduces a latent variable \(z_i\) for each 3D point \(p_i \in \mathbb{R}^3\). This latent variable is not a supervised target. It is an inferred physical semantic state: the hidden representation that explains how the point appears across views. Since this state is unobserved and may be ambiguous, we infer a posterior distribution rather than a single deterministic embedding:
\begin{equation}
q_\phi(z_i \mid \mathcal{C}_i)
=
\mathcal{N}
\left(
z_i;
\mu_i,
\operatorname{diag}(\sigma_i^2)
\right).
 \end{equation}

The mean \(\mu_i  \in \mathbb{R}^d\) represents the most likely latent state given the available views. The scale \(\sigma_i  \in \mathbb{R}^d\) represents epistemic uncertainty about this state. When observations are dense and consistent, the posterior can become sharp. When observations are sparse, noisy, or contradictory, the posterior remains uncertain. This uncertainty is not an auxiliary output; it is a direct consequence of the fact that the true latent 3D feature is never observed.

\subsection{Learning Without a Ground-Truth 3D Latent Space}
\begin{figure*}[t]
\centering

\newcommand{\topfigheight}{4.cm}
\newcommand{\bottomfigheight}{4.0cm}

\begin{subfigure}[b]{0.32\textwidth}
    \centering
    \includegraphics[
        width=\linewidth,
        height=\topfigheight,
        keepaspectratio
    ]{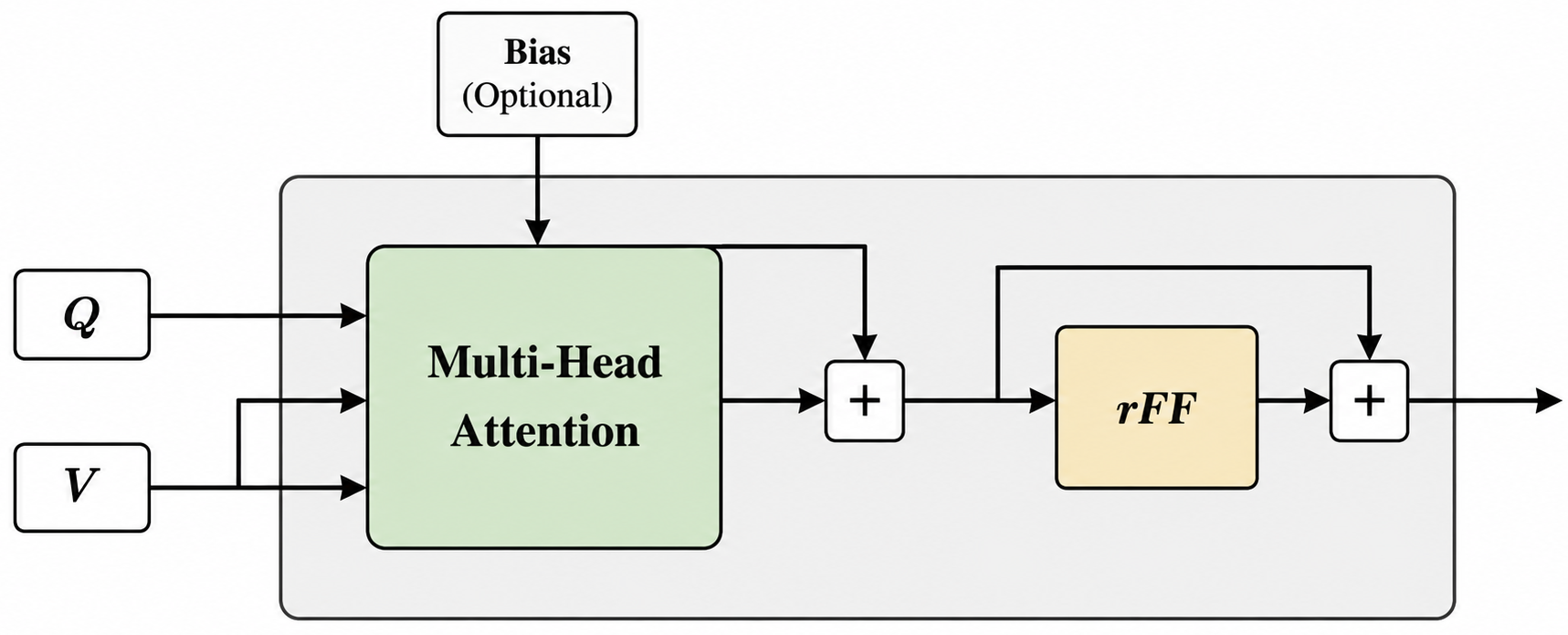}
    \caption{Multihead Attention Block (MAB).}
    \label{fig:mab}
\end{subfigure}
\hfill
\begin{subfigure}[b]{0.32\textwidth}
    \centering
    \includegraphics[
        width=\linewidth,
        height=\topfigheight,
        keepaspectratio
    ]{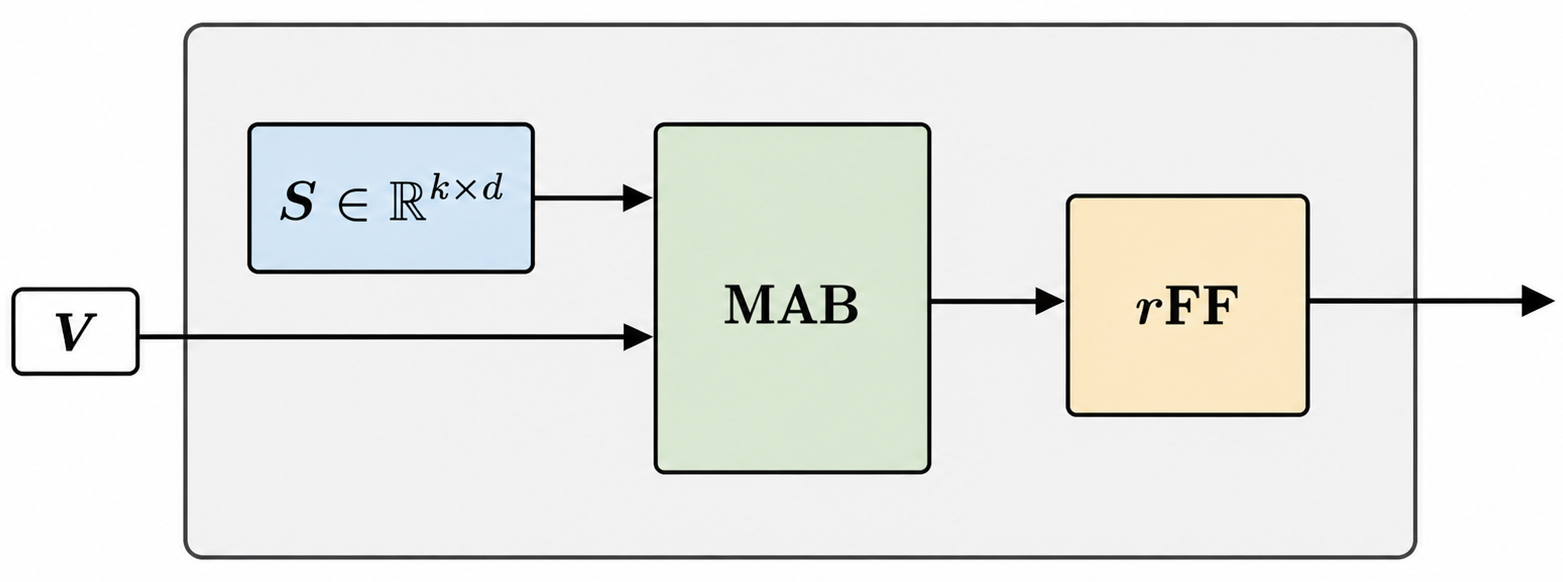}
    \caption{Pooling by Multihead Attention (PMA). $S$: learned seed vectors.}
    \label{fig:pma}
\end{subfigure}
\hfill
\begin{subfigure}[b]{0.32\textwidth}
    \centering
    \includegraphics[
        width=\linewidth,
        height=\topfigheight,
        keepaspectratio
    ]{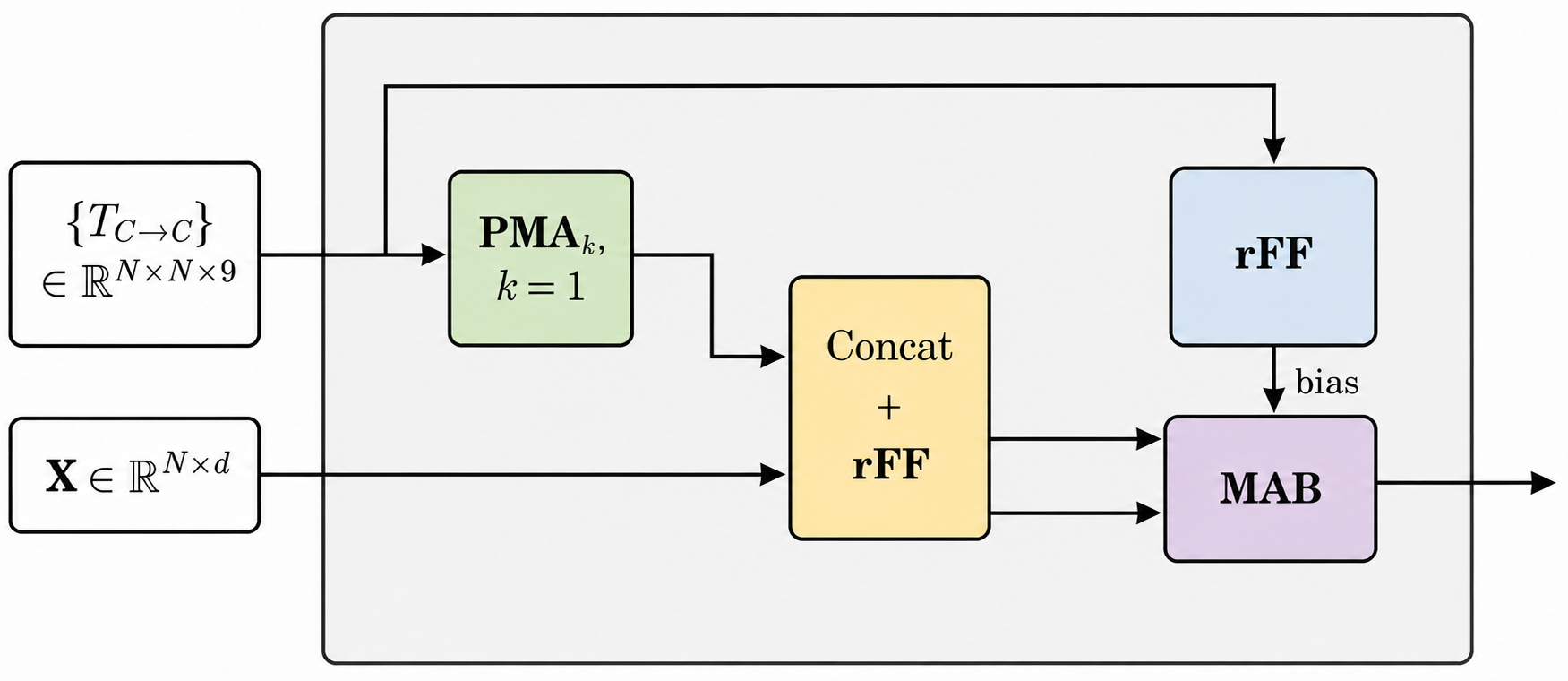}
    \caption{Our new Pose-aware Attention Block (PAB)}
    \label{fig:pab}
\end{subfigure}

\vspace{0.cm}

\begin{subfigure}[b]{0.49\textwidth}
    \centering
    \includegraphics[
        width=\linewidth,
        height=\bottomfigheight,
        keepaspectratio
    ]{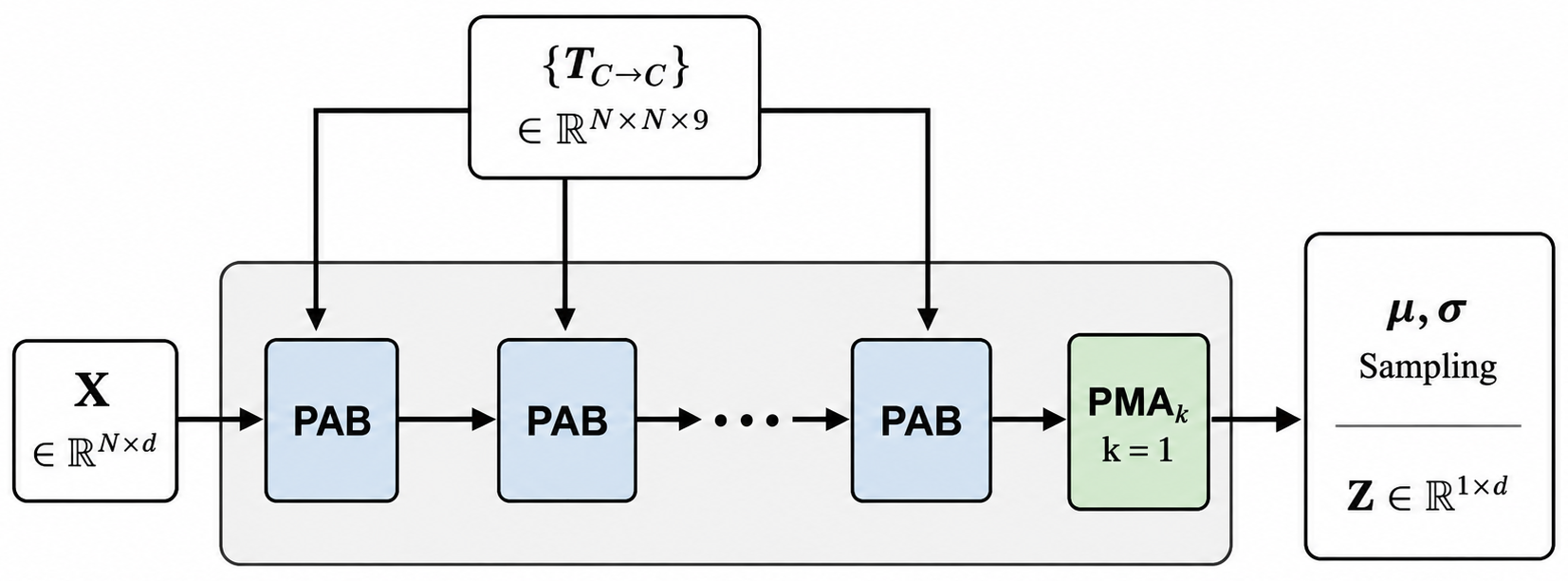}
    \caption{Our Encoder}
    \label{fig:encoder}
\end{subfigure}
\hfill
\begin{subfigure}[b]{0.49\textwidth}
    \centering
    \includegraphics[
        width=\linewidth,
        height=\bottomfigheight,
        keepaspectratio
    ]{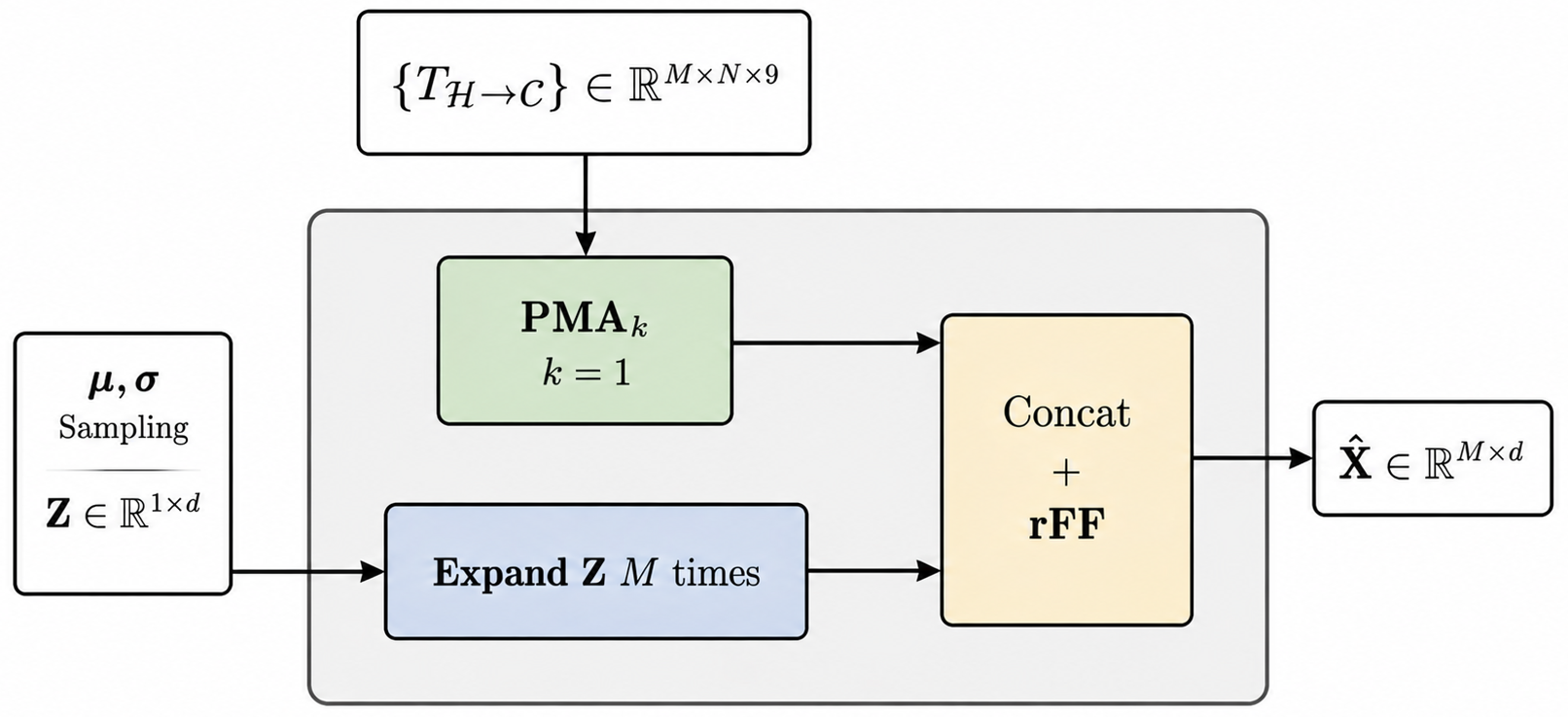}
    \caption{Our Decoder}
    \label{fig:decoder}
\end{subfigure}

\caption{Overview of our proposed architecture. More details in Appendix.}
\label{fig:architecture-blocks}
\end{figure*}

A key challenge is that we cannot supervise \(z_i\) directly. There is no dataset that provides the true CLIP-like or DINO-like feature of a physical 3D point independent of all views. The latent space must therefore be identified indirectly. \MyTitle\ learns this space through a simple physical principle: a correct latent state should predict measurements of the same point from views that were not used to infer it.

During training, we split the measurements of a point into a context set \(\mathcal{C}_i\) and a held-out set \(\mathcal{H}_i\). The context set \(\mathcal{C}_i\) is used to infer the latent state \(z_i\) through the posterior \(q_\phi(z_i \mid \mathcal{C}_i)\). The held-out set \(\mathcal{H}_i\) is used to test whether this inferred state can explain unseen measurements \((x_{ik},T_k)\) with \(T_k\in SE(3)\) the camera-to-world pose of that view. This provides a self-supervised training signal without requiring ground-truth 3D features \(z_i^\star\).
The encoder estimates \(q_\phi(z_i \mid \mathcal{C}_i)\).
The decoder receives a sample from this posterior and a target viewing condition, and predicts the feature that should be observed from that view: \(\hat{x}_{ik}=\mathrm{Decoder}_\theta(z_i,T_k)\).

Thus, \MyTitle\ is trained by cross-view prediction rather than input reconstruction. The model succeeds only if the latent state inferred from some views contains information that generalizes to other views of the same physical point.

We model each held-out 2D feature as a noisy, view-conditioned measurement of the latent 3D state:
\[
x_{ik}
\sim
p_\theta
\left(
x \mid z_i, T_k
\right).
\]
Thus, different views of the same point may induce different observed features due to viewpoint, scale, visibility, occlusion, and feature-extractor uncertainty.

\subsection{Training Objective}

For clarity, we write the held-out set as:
\(
\mathcal{H}_i
=
\left\{
\left(x_{ik},T_k\right)
\right\}_{k=1}^{N_i},
\qquad
N_i = |\mathcal{H}_i|.
\)
We also separate the held-out features and poses:
\(
\mathcal{X}_i^{\mathcal{H}}
=
\left\{
x_{ik}
\right\}_{k=1}^{N_i},
\qquad
\mathcal{T}_i^{\mathcal{H}}
=
\left\{
T_k
\right\}_{k=1}^{N_i}.
\)
Our goal is to explain held-out observations from the context observations and the held-out poses.
Because the latent state is unknown, we marginalize over possible latent explanations:
\begin{equation}
\log
p_\theta
\left(
\mathcal{X}_i^{\mathcal{H}}
\mid
\mathcal{C}_i,
\mathcal{T}_i^{\mathcal{H}}
\right)
=
\log
\int
p_\theta
\left(
\mathcal{X}_i^{\mathcal{H}}
\mid
z_i,
\mathcal{C}_i,
\mathcal{T}_i^{\mathcal{H}}
\right)
p(z_i)
\,\mathrm{d}z_i .
\end{equation}

In practice, we optimize the negative ELBO:
\begin{equation}
\mathcal{L}_{\mathrm{ELBO}}(i)
=
\mathcal{L}_{\mathrm{heldout}}(i)
+
\beta\,
\mathrm{KL}
\left[
q_\phi
\left(
z_i \mid \mathcal{C}_i
\right)
\,\middle\|\,
p(z_i)
\right].
\end{equation}

The held-out term is:
\begin{equation}
\mathcal{L}_{\mathrm{heldout}}(i)
=
-
\mathbb{E}_{z_i \sim q_\phi(\cdot \mid \mathcal{C}_i)}
\left[
\sum_{k=1}^{N_i}
\log
p_\theta
\left(
x_{ik}
\mid
z_i,
\mathcal{C}_i,
T_k
\right)
\right].
\end{equation}

In our implementation, we use a cosine loss between the predicted held-out feature \(\hat{x}_{ik}\) and the observed feature \(x_{ik}\):
\begin{equation}
\mathcal{L}_{\mathrm{heldout}}(i)
=
\frac{1}{N_i}
\sum_{k=1}^{N_i}
\left[
1
-
\cos
\left(
\hat{x}_{ik},
x_{ik}
\right)
\right].
\end{equation}

\subsection{Difference from a Classical VAE}

Although \MyTitle\ uses variational inference, it is not a classical VAE applied to 2D features. In a standard VAE, the encoder observes an input and the decoder reconstructs that same input:
\(
x \rightarrow z \rightarrow \hat{x}.
\)
The latent variable is therefore encouraged to preserve the information needed to reproduce the observed sample.\MyTitle\ changes the role of the latent variable. The encoder receives a set of context measurements, but the decoder must predict different measurements from held-out views:
\begin{equation}
\mathcal{C}_i \rightarrow z_i,
\qquad
(z_i,T_k) \rightarrow \hat{x}_{ik},
\qquad
(x_{ik},T_k)\notin \mathcal{C}_i.
\end{equation}

This distinction is central. If the model reconstructed the context features, it could simply encode view-specific artifacts: the particular camera pose, occlusion pattern, scale, or local image context. By predicting held-out views, \MyTitle\ forces the latent variable to capture what is stable across the measurement process. The target view \(T_k\) is given to the decoder, so view-dependent variation can be explained there rather than stored in the latent state.

In this sense, \MyTitle\ turns the VAE objective into a physical posterior-predictive problem. The latent variable is useful only if it explains measurements of the same 3D point that were not used to infer it.

Our network architecture differs from typical SetTransformers \cite{lee2019set, kim2021setvae} because we make the attention blocks camera-pose-aware, as shown in Figure ~\ref{fig:architecture-blocks}.
Let us define \(N\) as the cardinality of the context set \(C\) (we omit the index \(i \) for simplicity). Instead of considering camera to world poses \(T_{j}\), we consider relative poses. Let us define
\(
 T_{l \rightarrow k}
=
T_{k}^{-1}T_{l} 
\) Coordinate of view frame \(l\) expressed in coordinate of view frame \(k\). We can then define the tensor of relative poses over context views:
\begin{equation}
T_{\mathcal{C}\rightarrow\mathcal{C}} = \{T_{l \rightarrow k}\}_{l, k} \in SE(3)^{N \times N}
\end{equation}
Hence, \(T_{\mathcal{C}\rightarrow\mathcal{C}} \in \mathbb{R}^{N \times N \times 12}\) because 
each
relative pose is a \(3 \times 4\) matrix. We inject this relative pose information both in the query and in the bias of the multi-head attention as explained in Figure~\ref{fig:architecture-blocks}.

Similarly, instead of considering the camera to world poses of held-out observations when decoding, we consider the relative poses of each held-out view with respect to each context view: \(T_{\mathcal{H}\rightarrow \mathcal{C}}\), we concatenate this pose information with the latent space \(Z\) to predict the held-out 2D feature as displayed in Figure~\ref{fig:architecture-blocks}.

\begin{definition}[Permutation Equivariance and Invariance]
Let $\mathfrak{S}_n$ be the symmetric group on $n$ elements, and let
$x=(x_1,\dots,x_n) \in \mathcal{X}^n$. For any $\pi \in \mathfrak{S}_n$,
define the permuted input as
$\pi \cdot x = (x_{\pi(1)},\dots,x_{\pi(n)})$.

A function $f:\mathcal{X}^n \to \mathcal{Y}^n$ is
\emph{permutation equivariant} if
\begin{equation}
    f(\pi \cdot x) = \pi \cdot f(x),
    \qquad \forall \pi \in \mathfrak{S}_n,\ \forall x \in \mathcal{X}^n.
\end{equation}
A function $g:\mathcal{X}^n \to \mathcal{Y}$ is
\emph{permutation invariant} if
\begin{equation}
    g(\pi \cdot x) = g(x),
    \qquad \forall \pi \in \mathfrak{S}_n,\ \forall x \in \mathcal{X}^n.
\end{equation}
\end{definition}
\begin{proposition}[Permutation Equivariance of the PAB Block]
Let $\mathrm{PAB}:\mathcal{X}^n \to \mathcal{Y}^n$ denote the proposed
Pose-aware Attention Block. Then, for any permutation $\pi \in \mathfrak{S}_n$ and any
input $X \in \mathcal{X}^n$, the PAB block is permutation equivariant:
\[
    \mathrm{PAB}(\pi \cdot X)
    =
    \pi \cdot \mathrm{PAB}(X).
\]
\end{proposition}

\begin{proof}
The proof is provided in Appendix.
\end{proof}

\begin{proposition}[Permutation Invariance of the Encoder]
Let $\mathrm{Enc}:\mathcal{X}^n \to \mathcal{Y}$ denote the proposed
encoder. Then, for any permutation $\pi \in \mathfrak{S}_n$ and any
input $X \in \mathcal{X}^n$, the encoder is permutation invariant:
\[
    \mathrm{Enc}(\pi \cdot X)
    =
    \mathrm{Enc}(X).
\]
\end{proposition}

\begin{proof}
The proof is provided in Appendix.
\end{proof}

\subsection{Language-Compatible 3D Features}

For downstream task such as open-vocabulary segmentation with a CLIP VLM for instance, the inferred 3D features must remain compatible with text embeddings. So, we map it back to the original foundation-model embedding space:
\begin{equation}
e_i = h_\psi(z_i).
\end{equation}

We therefore add a text-prototype likelihood. Given a text prototype \(s_c\), we define:
\begin{equation}
\label{eq:text_prototype}
p_\psi(c_i \mid z_i)
=
\frac{
\exp
\left(
\cos(e_i,s_c)/\tau
\right)
}{
\sum_{c'=1}^{K}
\exp
\left(
\cos(e_i,s_{c'})/\tau
\right)
}.
\end{equation}

The text loss is:
\begin{equation}
\mathcal{L}_{\mathrm{text}}(i)
=
-
\log p_\psi(c_i \mid z_i).
\end{equation}

This term does not define the latent state by itself. Instead, it aligns the inferred physical representation with the language-aligned space of the pretrained model. The latent state is still primarily learned through cross-view measurement prediction.
At inference time, we use: \(e_i = h_\psi(\mu_i)\)

\subsection{Weak Anchoring to the Original Feature Space}

We also use the average lifted feature as a weak anchor:
\(a_i =
\sum_{j=1}^{M_i} w_{ij} x_{ij}.\)

The alignment loss is: \(\mathcal{L}_{\mathrm{avg}}(i)
=
1 -
\cos(e_i,a_i).\)

This term should not be interpreted as assuming that the average feature is the correct 3D representation. Rather, it keeps the inferred feature from drifting too far from the original pretrained feature space. The average feature acts as a weak measurement, while the latent state is learned through posterior-predictive consistency.

\subsection{Full Objective}

The full training loss is:
\begin{equation}
\mathcal{L}(i)
=
\mathcal{L}_{\mathrm{view}}(i)
+
\beta\,
\mathrm{KL}
\left[
q_\phi(z_i \mid \mathcal{C}_i)
\,\middle\|\,
p(z_i)
\right]
+
\lambda_{\mathrm{text}}\,
\mathcal{L}_{\mathrm{text}}(i)
+
\lambda_{\mathrm{avg}}\,
\mathcal{L}_{\mathrm{avg}}(i).
\end{equation}

The role of each term is distinct. The view prediction term learns the latent state by asking it to explain unseen measurements. The KL term models uncertainty and regularizes the unobserved latent space. The text term preserves semantic compatibility with language. The average-alignment term softly anchors the output to the original feature space.

Together, these terms allow \MyTitle\ to learn a 3D feature field without ever observing ground-truth 3D features.

\subsection{Inference}

At test time, all available measurements are used as context. The encoder estimates:
\begin{equation}
(\mu_i,\sigma_i)
=
\mathrm{Encoder}_\phi(\mathcal{C}_i).
\end{equation}

The final 3D feature is:
\begin{equation}
\label{eq:posterior}
e_i = h_\psi(\mu_i).
\end{equation}

For open-vocabulary segmentation, each point is assigned to the closest text prototype:
\begin{equation}
\hat{c}_i
=
\arg\max_c
\cos(e_i,s_c).
\end{equation}

The posterior scale \(\sigma_i\) provides an uncertainty estimate over the inferred physical state. This uncertainty is especially useful in regions with sparse views, inconsistent observations, object boundaries, or strong occlusions.

\section{Experiments}
\label{sec:experiments}

We evaluate \MyTitle\ on \emph{annotation-free open-vocabulary 3D semantic segmentation}. 
Our goal is to test whether the posterior-inferred 3D semantic field learned by \MyTitle\ is more semantically coherent than direct deterministic 2D-to-3D feature lifting.
Given a reconstructed point cloud, we reproject each 3D point into multiple images, collect the corresponding 2D semantic features and relative poses, and infer a latent 3D embedding using the posterior in Eq.~\ref{eq:posterior}. 
At test time, each 3D point is classified by similarity between its inferred embedding and text prototypes. 
This setup directly measures whether probabilistic multi-view semantic inference improves 3D semantic organization beyond simple aggregation.

\subsection{Experimental Setup}

\paragraph{Datasets.}
We evaluate on three standard indoor 3D benchmarks. 
\textbf{ScanNet}~\cite{dai2017scannet} contains diverse RGB-D indoor scenes with \(1{\,}201\) training and \(312\) validation scans, together with depth, camera poses, and intrinsics that enable multi-view 2D--3D lifting. 
Following common practice, we map the ScanNet label \textit{other furniture} to \textit{other}. 
We also evaluate on \textbf{ScanNet200}, which extends ScanNet to 200 semantic categories and substantially increases vocabulary size and long-tail difficulty. 
Finally, we report results on \textbf{Matterport3D}~\cite{chang2017matterport3d}, a large-scale RGB-D dataset with complex building interiors and denser scene coverage. 
Together, these datasets test both standard indoor open-vocabulary segmentation and robustness under larger semantic vocabularies.

\paragraph{Task and evaluation protocol.}
For each 3D point, we gather all valid reprojected 2D observations, infer a latent 3D semantic embedding, and classify the point by similarity to text embeddings. 
We report \textbf{mean Intersection-over-Union (mIoU)} and \textbf{mean Accuracy (mAcc)}, the standard metrics for open-vocabulary 3D segmentation. 
Unless otherwise specified, all text prototypes are built from the same prompt set across methods.

\paragraph{Implementation details.}
We retain only geometrically valid reprojections satisfying
\(
|z - d(u,v)| < \delta\, d(u,v)
\),
where \(z\) is the reprojected depth and \(d(u,v)\) is the observed pixel depth; we use \(\delta = 0.05\). 
We train with learning rate \(5 \times 10^{-5}\), batch size \(512\), cosine decay, and \(5{\,}000\) warmup steps. Training is done on 1 NVIDIA RTX A6000 GPU. 
The variational objective uses KL annealing for \(100{\,}000\) steps, with \(\beta = 10^{-4}\), \(\lambda_{\text{text}}=0.1\), and \(\lambda_{\text{avg}}=0.05\). 
Each epoch contains \(10{\,}000\) steps and training runs for \(50\) epochs. 
We use up to \(40\) observations per point, with at most \(20\) context views and \(20\) held-out views; missing observations are handled by padding and masking. 
For text prototype matching in Eq.~\ref{eq:text_prototype}, we use temperature \(0.05\). 
Unless otherwise noted, our default 2D semantic teacher is TIPSv2~\cite{cao2026tipsv2}. 
Additional implementation details are provided in the Appendix.

\subsection{Main Results: Annotation-Free Open-Vocabulary 3D Segmentation}

Table~\ref{tab:main_annotation_free_3d_segmentation} compares \MyTitle\ against three families of methods: 
(i) zero-shot 2D feature fusion methods, 
(ii) 2D fusion methods with training, and 
(iii) methods that additionally train a 3D backbone.
Our method consistently improves over prior annotation-free approaches across all datasets. 
Notably, \MyTitle\ substantially outperforms direct zero-shot lifting methods, showing that strong 2D semantic features alone are not sufficient to obtain a stable 3D semantic field. 
This supports our central hypothesis: 3D semantics should be inferred as a latent variable from noisy multi-view evidence rather than obtained by direct feature projection or voting.

A second important observation is that \MyTitle\ is already highly competitive \emph{without} 3D backbone distillation. 
This isolates the gain from probabilistic multi-view semantic inference itself. 
When combined with optional 3D distillation, performance further improves on all datasets, increasing the margin over previous 3D-training methods and narrowing the gap to full supervision. 
These results indicate that \MyTitle\ provides a strong 3D semantic representation on its own, while remaining complementary to downstream 3D backbone training. Further qualitative results are given in Appendix, Figures~\ref{fig:gt_pred_qualitative_1} and \ref{fig:gt_pred_qualitative_2}.

\begin{table*}[t]
\centering
\small
\setlength{\tabcolsep}{5.2pt}
\renewcommand{\arraystretch}{1.13}
\begin{tabular}{l c cc cc cc}
\toprule
\multirow{2}{*}{\textbf{Method}}
&
\multirow{2}{*}{\makecell[c]{\textbf{Test-Time}\\\textbf{Images}}}
&
\multicolumn{2}{c}{\textbf{ScanNet}}
&
\multicolumn{2}{c}{\textbf{ScanNet200}}
&
\multicolumn{2}{c}{\textbf{Matterport3D}}
\\
\cmidrule(lr){3-4}
\cmidrule(lr){5-6}
\cmidrule(lr){7-8}
&
&
\textbf{mIoU} & \textbf{mAcc}
&
\textbf{mIoU} & \textbf{mAcc}
&
\textbf{mIoU} & \textbf{mAcc}
\\
\midrule

\multicolumn{8}{l}{\textit{2D fusion Zero-shot}} \\
MaskCLIP-3D~\cite{lambert2020mseg}
& \checkmark
& 9.7 & 21.6
& -- & --
& -- & -- \\

MSeg Voting~\cite{lambert2020mseg}
& \checkmark
& 45.6 & 54.4
& -- & --
& 33.4 & 39.0 \\

OpenScene-2D~\cite{peng2023openscene}
& \checkmark
& 50.0 & 62.7
& -- & --
& 32.3 & 40.0 \\

CLIP-FO3D, feature projection~\cite{zhang2023clip}
& \checkmark
& 27.6 & 47.7
& -- & --
& -- & -- \\

\midrule
\multicolumn{8}{l}{\textit{2D fusion w/ training}} \\
\rowcolor{gray!12}
\textbf{Ours}
& \checkmark
& \textbf{64.6} & \textbf{75.5}
& \textbf{11.0} & \textbf{18.1}
& \textbf{50.4} & \textbf{65.1} \\

\midrule
\multicolumn{8}{l}{\textit{3D point-cloud training}} \\
OpenScene + 3D Distillation~\cite{peng2023openscene}
& \xmark
& 52.9 & 63.2
& 7.3 & --
& 41.9 & 51.2 \\

PLA~\cite{ding2023pla}
& \xmark
& -- & --
& 1.8 & --
& -- & -- \\

RegionPLC~\cite{yang2024regionplc}
& \xmark
& -- & --
& 6.5 & --
& -- & -- \\

OV3D~\cite{jiang2024open}
& \xmark
& 57.3 & 72.9
& 8.7 & --
& 45.8 & 62.4 \\

CLIP-FO3D~\cite{zhang2023clip}
& \xmark
& 30.2 & 49.1
& -- & --
& -- & -- \\

PGOV3D~\cite{zhang2025pgov3d}
& \xmark
& 59.5 & 73.2
& 9.3 & 17.1
& -- & -- \\

\rowcolor{gray!12}
\textbf{Ours + 3D Distillation}
& \xmark
& \textbf{65.8} & \textbf{76.7}
& \textbf{14.2} & \textbf{22.6}
& \textbf{53.2} & \textbf{66.6} \\

\midrule
\multicolumn{8}{l}{\textit{Fully supervised upper bound}} \\
Fully supervised
& \xmark
& 72.0 & 80.7
& 23.9 & 32.9
& 55.7 & 67.4 \\

\bottomrule
\end{tabular}
\caption{
Annotation-free open-vocabulary 3D semantic segmentation on ScanNet, ScanNet200, and Matterport3D.
\MyTitle\ consistently improves over prior zero-shot and trained 2D-fusion methods, and remains competitive with methods that additionally optimize a 3D backbone.
}
\label{tab:main_annotation_free_3d_segmentation}
\end{table*}

\subsection{Ablations and Analysis}

\paragraph{Backbone-agnostic semantic inference.}
Table~\ref{tab:vlm_miou_comparison} shows that \MyTitle\ is not tied to a specific 2D semantic backbone. 
Replacing the underlying 2D foundation model consistently changes the quality of the inferred 3D field in the expected direction: stronger 2D semantic encoders yield stronger 3D representations. 
This confirms that \MyTitle\ can exploit improvements in foundation-model quality without modifying the 3D inference architecture.

\paragraph{Latent structure versus deterministic averaging.}
In Appendix, Figure~\ref{fig:tsne_avg_vs_latent_2x2} compares the semantic geometry induced by simple average aggregation and by the latent representation inferred by \MyTitle \ .Across multiple representative scenes, latent embeddings exhibit tighter class-local neighborhoods and markedly reduced inter-class overlap, while average aggregation leads to fragmented clusters and stronger semantic mixing. 

\paragraph{Robustness to sparse views.}
Figure~\ref{fig:view_robustness} 
shows that our method  remains substantially more stable as the number of available views decreases, even in the highly sparse setting with only \(4\) views. In contrast, directly averaging 2D features degrades rapidly as view coverage becomes sparse, highlighting the benefit of our inferred 3D feature field.

\paragraph{Efficiency.}
Table~\ref{tab:efficiency_analysis} reports model size and runtime statistics. 
FLOPs, inference latency, and peak GPU memory are measured for a single forward pass with batch size \(1\), \(20\) context observations, and \(20\) held-out observations; inference latency is averaged over \(100\) runs after \(10\) warmup iterations on a single GPU. 
\MyTitle\ remains lightweight, requiring only \(0.64\)G FLOPs, \(14.4\) ms inference time, and \(150\) MB peak GPU memory. 
This makes it practical as a semantic inference module and complementary to heavier 3D distillation pipelines.

\begin{table}[t]
\centering
\small
\setlength{\tabcolsep}{4.5pt}
\renewcommand{\arraystretch}{1.12}

\begin{subtable}[t]{0.48\columnwidth}
\centering
\begin{tabular}{lcc}
\toprule
\textbf{2D Feature} 
& \textbf{ScanNet mIoU} 
& \textbf{$\Delta$} \\
\midrule
MaskCLIP~\cite{zhou2022extract} 
& 60.8 
& -- \\
CLIP-DINOiser~\cite{wysoczanska2024clip} 
& 62.3 
& $+1.5$ \\
OpenSeg~\cite{ghiasi2022scaling} 
& 63.4 
& $+2.6$ \\
\rowcolor{gray!12}
TIPSv2~\cite{cao2026tipsv2} 
& \textbf{64.6} 
& \textbf{$+3.8$} \\
\bottomrule
\end{tabular}
\caption{\textbf{Performance with different 2D VLM backbones.}}
\label{tab:vlm_miou_comparison}
\end{subtable}
\hfill
\begin{subtable}[t]{0.48\columnwidth}
\centering
\begin{tabular}{lc}
\toprule
\textbf{Metric} 
& \textbf{Value} \\
\midrule
FLOPs $\downarrow$ 
& 0.64G \\
Inference Time $\downarrow$ 
& 14.4 ms \\
Training Time $\downarrow$ 
& 41.6 h \\
Parameters $\downarrow$ 
& 35.6M \\
Peak GPU Memory $\downarrow$ 
& 150 MB \\
\bottomrule
\end{tabular}
\caption{\textbf{Efficiency.}}
\label{tab:efficiency_analysis}
\end{subtable}

\vspace{3pt}
\caption{
\MyTitle\ can benefit from stronger 2D foundations models while remaining compact and efficient.
}
\label{tab:backbone_and_efficiency}
\end{table}

\begin{figure}[t]
    \centering
    \includegraphics[width=0.48\textwidth]{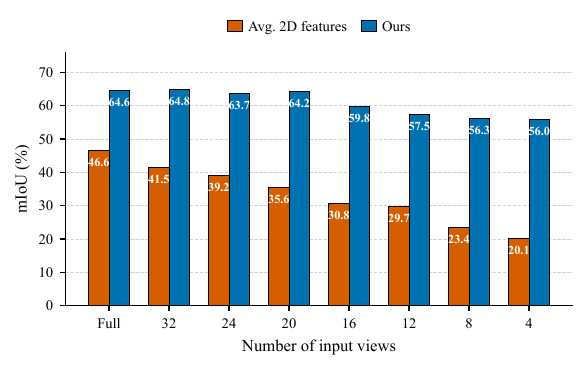}
    \caption{
    Robustness to sparse views. Compared with directly averaging 2D features across views.
    }
    \label{fig:view_robustness}
\end{figure}

\section{Conclusion}
\label{sec:conclusion}
Casper3D frames 3D semantic fusion as probabilistic inference over noisy multi-view foundation-model features. By combining set-based variational inference, relative-pose-aware reasoning, and held-out observation prediction, it learns a more stable 3D semantic representation than naive multi-view aggregation. The method remains compatible with open-vocabulary language supervision and different geometry providers, making it a simple and general approach for self-supervised 3D semantic understanding. \paragraph{Limitations and future work.}
A current limitation of Casper3D is that it relies on ground-truth 3D geometry to reproject features across 2D views during training and evaluation or assumes access to accurate scene geometry. 
A natural next step is to replace this requirement with off-the-shelf 3D estimation models, feed-forward 3D scene prediction methods, enabling Casper3D to operate in settings where ground-truth 3D geometry is unavailable.

\newpage
{
\small
\bibliographystyle{abbrv}
\bibliography{main}
}

\clearpage
\setcounter{proposition}{0}
\appendix

\section{Proofs}

\begin{lemma}[MAB is Equivariant with respect to its query and Invariant with respect to its value].
\label{lem:mab-invariant}
\end{lemma}
Let $\mathrm{MAB}(Q,V)$ be the Multihead Attention Block of Set Transformer,
where $Q\in\mathbb{R}^{n\times d}$ is the query input and
$V\in\mathbb{R}^{m\times d}$ is the key-value input. Let
$\pi\in\mathfrak{S}$ be a permutation. Then
\[
    \mathrm{MAB}(\pi\cdot Q,V)
    =
    \pi\cdot \mathrm{MAB}(Q,V),
\]
and
\[
    \mathrm{MAB}(Q,\pi\cdot V)
    =
    \mathrm{MAB}(Q,V).
\]

This Lemma results from the Multi-head attention equation and is proved in \cite{lee2019set}.

\begin{lemma}[Query Equivariance of MAB with Query-Permuted Bias]
\label{lem:mab-biased-equivariant}
Let $\mathrm{MAB}_{B}(Q,V)$ denote a Multihead Attention Block with additive
attention bias $B\in\mathbb{R}^{n\times m}$, where
$Q\in\mathbb{R}^{n\times d}$ is the query input and
$V\in\mathbb{R}^{m\times d}$ is the key-value input. Let
$\pi\in\mathfrak{S}_n$ be a permutation of the query indices. If the bias is
permuted in the same way as the queries, then
\[
    \mathrm{MAB}_{\pi\cdot B}(\pi\cdot Q,V)
    =
    \pi\cdot \mathrm{MAB}_{B}(Q,V).
\]
\end{lemma}

\begin{proof}
It suffices to prove the result for a single attention head, since multi-head
attention is obtained by concatenating several heads and applying shared linear
maps, which preserve permutation equivariance.

For one attention head with additive bias $B$, define
\[
    \operatorname{Att}_{B}(Q,V)
    =
    \operatorname{softmax}\!\left(
        \frac{QW_Q(VW_K)^\top}{\sqrt{d}}
        +
        B
    \right)VW_V .
\]
Let
\[
    A
    :=
    \frac{QW_Q(VW_K)^\top}{\sqrt{d}}
    +
    B .
\]
Then, after permuting the query input and the bias by the same permutation
$\pi$, the attention logits become
\[
\begin{aligned}
    \frac{(\pi\cdot Q)W_Q(VW_K)^\top}{\sqrt{d}}
    +
    \pi\cdot B
    &=
    \pi\cdot
    \frac{QW_Q(VW_K)^\top}{\sqrt{d}}
    +
    \pi\cdot B \\[1mm]
    &=
    \pi\cdot
    \left(
        \frac{QW_Q(VW_K)^\top}{\sqrt{d}}
        +
        B
    \right) \\[1mm]
    &=
    \pi\cdot A .
\end{aligned}
\]
Since the softmax is applied row-wise, it commutes with row permutations:
\[
    \operatorname{softmax}(\pi\cdot A)
    =
    \pi\cdot \operatorname{softmax}(A).
\]
Therefore,
\[
\begin{aligned}
    \operatorname{Att}_{\pi\cdot B}(\pi\cdot Q,V)
    &=
    \operatorname{softmax}\!\left(
        \frac{(\pi\cdot Q)W_Q(VW_K)^\top}{\sqrt{d}}
        +
        \pi\cdot B
    \right)VW_V \\[1mm]
    &=
    \operatorname{softmax}(\pi\cdot A)VW_V \\[1mm]
    &=
    \bigl(\pi\cdot \operatorname{softmax}(A)\bigr)VW_V \\[1mm]
    &=
    \pi\cdot
    \left(
        \operatorname{softmax}(A)VW_V
    \right) \\[1mm]
    &=
    \pi\cdot \operatorname{Att}_{B}(Q,V).
\end{aligned}
\]
Thus, a single biased attention head is permutation equivariant with respect to
the query input when the bias is permuted in the same way.

The multi-head attention layer preserves this property because each head
satisfies the same equivariance relation and the output projection is applied
row-wise with shared parameters. The remaining operations in $\mathrm{MAB}$,
namely residual connections, layer normalization, and row-wise feed-forward
networks, are also applied identically to each row and therefore commute with
the action of $\pi$. Hence,
\[
    \mathrm{MAB}_{\pi\cdot B}(\pi\cdot Q,V)
    =
    \pi\cdot \mathrm{MAB}_{B}(Q,V).
\]
\end{proof}

\begin{lemma}[Permutation Invariance of PMA]
\label{lem:pma-invariant}
Let $\mathrm{PMA}$ denote pooling by multihead attention with a learned seed
matrix $S$. Then, for any permutation $\pi\in\mathfrak{S}_n$ and any
$X\in\mathbb{R}^{n\times d}$,
\[
    \mathrm{PMA}(\pi\cdot X)
    =
    \mathrm{PMA}(X).
\]
\end{lemma}

\begin{proof}
Since $\mathrm{rFF}$ is applied row-wise with shared parameters, it is
permutation equivariant:
\[
    \mathrm{rFF}(\pi\cdot X)
    =
    \pi\cdot \mathrm{rFF}(X).
\]
Therefore,
\[
\begin{aligned}
    \mathrm{PMA}(\pi\cdot X)
    &= \mathrm{MAB}\bigl(S, \mathrm{rFF}(\pi\cdot X)\bigr) \\
    &= \mathrm{MAB}\bigl(S, \pi\cdot \mathrm{rFF}(X)\bigr) \\
    &= \mathrm{MAB}\bigl(S, \mathrm{rFF}(X)\bigr) \\
    &= \mathrm{PMA}(X).
\end{aligned}
\]
The third equality follows from Lemma 1.
\end{proof}

\begin{proposition}[Permutation Equivariance of the PAB Block]

\end{proposition}

\begin{proof}
Let $X\in\mathbb{R}^{N\times d}$ denote the set of multi-view features, and let
\[
    P(X) := \mathcal{T}_{\mathcal{C}\to\mathcal{C}}
    \in \mathbb{R}^{N\times N\times 9}
\]
denote the corresponding pairwise relative-pose tensor. We assume that the pose
tensor is permutation equivariant, i.e., for every permutation
$\pi\in\mathfrak{S}_N$,
\[
    P(\pi\cdot X)=\pi\cdot P(X),
\]
where the action of $\pi$ on $P(X)$ permutes the view indices consistently.

Define the intermediate pose-aware representation
\[
    \Phi(X)
    :=
    \mathrm{rFF}\!\left(
        \operatorname{Concat}\!\left(
            X,\mathrm{PMA}\!\left(P(X)\right)
        \right)
    \right).
\]
We first show that $\Phi$ is permutation equivariant. Using the equivariance of
$P$, Lemma~\ref{lem:pma-invariant}, and the row-wise equivariance of
$\mathrm{rFF}$, we obtain
\[
\begin{aligned}
    \Phi(\pi\cdot X)
    &=
    \mathrm{rFF}\!\left(
        \operatorname{Concat}\!\left(
            \pi\cdot X,
            \mathrm{PMA}\!\left(P(\pi\cdot X)\right)
        \right)
    \right) \\[1mm]
    &=
    \mathrm{rFF}\!\left(
        \operatorname{Concat}\!\left(
            \pi\cdot X,
            \mathrm{PMA}\!\left(\pi\cdot P(X)\right)
        \right)
    \right)
    && \text{by } P(\pi\cdot X)=\pi\cdot P(X) \\[1mm]
    &=
    \mathrm{rFF}\!\left(
        \operatorname{Concat}\!\left(
            \pi\cdot X,
            \mathrm{PMA}\!\left(P(X)\right)
        \right)
    \right)
    && \text{by Lemma~\ref{lem:pma-invariant}} \\[1mm]
    &=
    \pi\cdot
    \mathrm{rFF}\!\left(
        \operatorname{Concat}\!\left(
            X,
            \mathrm{PMA}\!\left(P(X)\right)
        \right)
    \right)
    && \text{by row-wise equivariance of $\mathrm{rFF}$} \\[1mm]
    &=
    \pi\cdot \Phi(X).
\end{aligned}
\]

By definition of the PAB block, shown in Figure~\ref{fig:pab},
\[
    \mathrm{PAB}(X)
    =
    \mathrm{MAB}_{P(X)}\!\left(\Phi(X),\Phi(X)\right),
\]
where $P(X)$ is used as the additive pose-dependent attention bias.

Therefore,
\[
\begin{aligned}
    \mathrm{PAB}(\pi\cdot X)
    &=
    \mathrm{MAB}_{P(\pi\cdot X)}
    \!\left(\Phi(\pi\cdot X),\Phi(\pi\cdot X)\right) \\[1mm]
    &=
    \mathrm{MAB}_{\pi\cdot P(X)}
    \!\left(\pi\cdot\Phi(X),\pi\cdot\Phi(X)\right)
    && \text{by } P(\pi\cdot X)=\pi\cdot P(X)
       \text{ and } \Phi(\pi\cdot X)=\pi\cdot\Phi(X) \\[1mm]
        &=
    \mathrm{MAB}_{\pi\cdot P(X)}
    \!\left(\pi\cdot\Phi(X),\Phi(X)\right)
    && \text{by Lemma~\ref{lem:mab-invariant}}  \\[1mm]
    &=
    \pi\cdot
    \mathrm{MAB}_{P(X)}
    \!\left(\Phi(X),\Phi(X)\right)
    && \text{by Lemma~\ref{lem:mab-biased-equivariant}} \\[1mm]
    &=
    \pi\cdot \mathrm{PAB}(X).
\end{aligned}
\]
Hence, $\mathrm{PAB}$ is permutation equivariant.
\end{proof}

\begin{proposition}[Permutation Invariance of the Encoder]
\label{prop:pab-encoder-invariance}
Let $\mathrm{Enc}$ be the encoder as described in Figure~\ref{fig:encoder} obtained by composing permutation-equivariant
PAB blocks, followed by a permutation-invariant PMA layer. Then, for any
permutation $\pi\in\mathfrak{S}_N$ and any input
$X\in\mathbb{R}^{N\times d}$,
\[
    \mathrm{Enc}(\pi\cdot X)
    =
    \mathrm{Enc}(X).
\]
\end{proposition}

\begin{proof}
Let the encoder be defined as
\[
    \mathrm{Enc}(X)
    :=
    \mathrm{PMA}\!\left(
        \mathrm{PAB}_L \circ \cdots \circ \mathrm{PAB}_1(X)
    \right),
\]
where each $\mathrm{PAB}_\ell$ is permutation equivariant. Define
\[
    H_0 := X,
    \qquad
    H_\ell := \mathrm{PAB}_\ell(H_{\ell-1}),
    \qquad
    \ell=1,\dots,L.
\]
Since each $\mathrm{PAB}_\ell$ is permutation equivariant, we have, by induction,
\[
    H_\ell(\pi\cdot X)
    =
    \pi\cdot H_\ell(X),
    \qquad
    \ell=1,\dots,L.
\]
In particular,
\[
    H_L(\pi\cdot X)
    =
    \pi\cdot H_L(X).
\]
Finally, since $\mathrm{PMA}$ is permutation invariant, we obtain
\[
\begin{aligned}
    \mathrm{Enc}(\pi\cdot X)
    &=
    \mathrm{PMA}\!\left(H_L(\pi\cdot X)\right) \\
    &=
    \mathrm{PMA}\!\left(\pi\cdot H_L(X)\right) \\
    &=
    \mathrm{PMA}\!\left(H_L(X)\right) \\
    &=
    \mathrm{Enc}(X).
\end{aligned}
\]
Therefore, the PAB encoder is permutation invariant.
\end{proof}

\section{Implementation Details}
\label{app:implementation_details}

\paragraph{2D semantic features.}
Our default 2D semantic teacher is OpenSeg. 
For each valid 2D reprojection of a 3D point, we extract a \(768\)-dimensional semantic feature and store it together with view identity and geometric metadata. 
In all main experiments reported in the paper, the Set-VAE operates on these OpenSeg features, although the architecture itself is backbone-agnostic.

\paragraph{Geometric preprocessing.}
We preprocess each scene by reprojecting 3D points into all available RGB-D views using camera intrinsics and extrinsics. 
A reprojected observation is considered valid only if it satisfies geometric and visibility constraints. 
We use a relative depth consistency rule between the reprojected depth and the observed depth map, together with a border exclusion heuristic to remove unstable image-edge projections. 
The resulting valid observations are grouped by 3D point. 
For each 3D point, we store: its 3D coordinates, the list of valid view ids, the corresponding 2D semantic features, and the associated camera extrinsics. 
We additionally precompute the average feature across valid observations for each point, which is used as a baseline and as an auxiliary alignment target during training.

\paragraph{Training samples.}
Training is performed on scene chunks. 
Each training sample is constructed by first selecting a scene, then sampling a chunk of \(2048\) 3D points from that scene. 
For each point, we sample up to \(20\) context views and up to \(20\) held-out target views. 
The fraction of observations allocated to context is capped by \texttt{max\_context\_frac}=0.6. 
When fewer observations are available, we use padding and binary masks. 
At training time, the model receives the context observations and is trained to predict the semantic embeddings of held-out observations.

\paragraph{Observation representation.}
Each context observation is represented by a \(1024\)-dimensional semantic feature. 
Pairwise relative pose features are computed between context views and are used as pairwise attention biases in the encoder. 
Held-out target views are represented by their relative pose with respect to the context views and are used by the decoder for conditional prediction.

\paragraph{Network Architecture.}
Our architecture takes as input a set of \(1024\)-dimensional semantic tokens. 
An input linear layer first maps these features to an encoder dimension of \(1024\). 
The encoder consists of \(4\) stacked self-attention blocks with \(8\) attention heads, LayerNorm, and dropout probability \(0.05\). 
Relative pose information is injected through a dedicated pose encoder that takes \(14\)-dimensional pose descriptors as input and maps them to attention-space biases using a hidden size of \(64\) and output size \(128\), with dropout \(0.05\).

After encoder processing, we reduce the set representation using a learned seed-based reducer with \(4\) seeds. 
This reduced representation is passed through \(2\) latent layers with latent dimension \(512\). 
The posterior latent variable has dimension \(512\). 
The decoder uses \(4\) attention heads, a hidden dimension of \(512\), \(2\) hidden layers, and dropout \(0.05\). 
A final projection head maps the latent representation back to CLIP-compatible semantic space using \(2\) hidden layers of size \(1024\).

\paragraph{Latent objective.}
We train the variational model with a held-out semantic reconstruction objective and a KL regularizer. 
The held-out reconstruction loss is computed in feature space using cosine similarity. 
The KL term is weighted by \(10^{-4}\) and linearly annealed over the first \(100{,}000\) optimization steps. 
We do not use free bits in the reported experiments.

\paragraph{Semantic alignment losses.}
In addition to the held-out reconstruction objective, we use two auxiliary semantic alignment terms. 
First, we align the predicted 3D semantic representation with the mean aggregated 2D semantic embedding of the same 3D point, using weight \(\lambda_{\text{avg}}=0.05\). 
Second, we align the inferred representation with text space using text prototypes, with weight \(\lambda_{\text{text}}=0.1\). 
Text similarity uses temperature \(0.05\). 
Unless otherwise noted, text prototypes are built with the LAION-style prompt templates.

\paragraph{Optimization.}
We train with Adam using learning rate \(5\times 10^{-5}\), \(\beta_1=0.9\), \(\beta_2=0.999\), and weight decay \(10^{-4}\). 
We use gradient clipping with maximum norm \(1.0\). 
The learning rate follows a warmup-cosine schedule with \(5{,}000\) warmup steps and minimum learning-rate ratio \(0.01\). 
Training runs for \(50\) epochs, with \(10{,}000\) optimization steps per epoch.

\paragraph{Batching and data loading.}
We use a batch size of \(512\) for training. 
Data loading uses \(16\) workers and prefetch factor \(4\). 
All experiments are run on a NVIDIA RTX A6000.

\paragraph{Voxelization and evaluation protocol.}
For 3D evaluation, we voxelize the point cloud using voxel size \(0.02\). 
This is used only for evaluation and comparison on standard 3D segmentation metrics. In the default evaluation setting, we use \(20\) context views per point.

\paragraph{Reproducibility.}
All architectural and optimization hyperparameters used in the main experiments are fixed in the public configuration file. 
In particular, the encoder depth, number of heads, latent dimension, number of context and target views, optimizer settings, KL annealing schedule, and alignment-loss weights are kept constant across the main ScanNet experiments.

\subsection{More results}
\begin{figure}[t]
\centering

\begin{subfigure}[t]{0.48\linewidth}
    \centering
    \includegraphics[width=\linewidth]{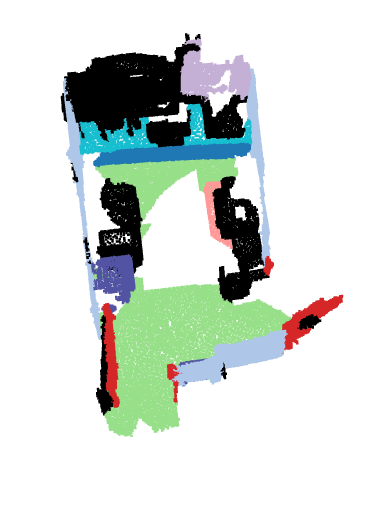}
\end{subfigure}
\hfill
\begin{subfigure}[t]{0.48\linewidth}
    \centering
    \includegraphics[width=\linewidth]{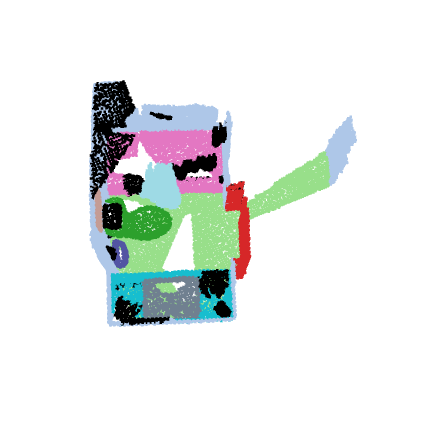}
\end{subfigure}

\vspace{4pt}

\begin{subfigure}[t]{0.48\linewidth}
    \centering
    \includegraphics[width=\linewidth]{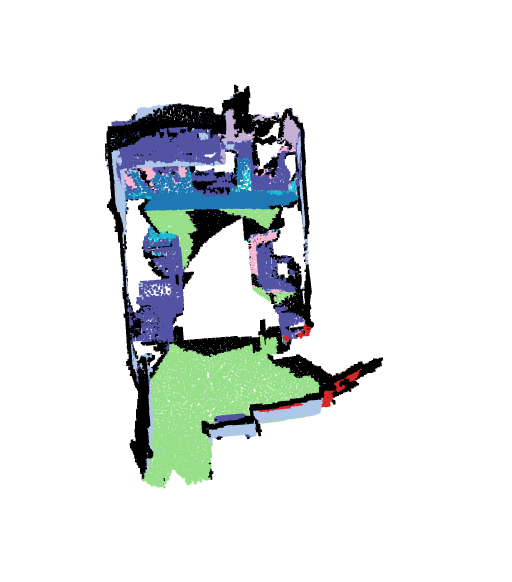}
\end{subfigure}
\hfill
\begin{subfigure}[t]{0.48\linewidth}
    \centering
    \includegraphics[width=\linewidth]{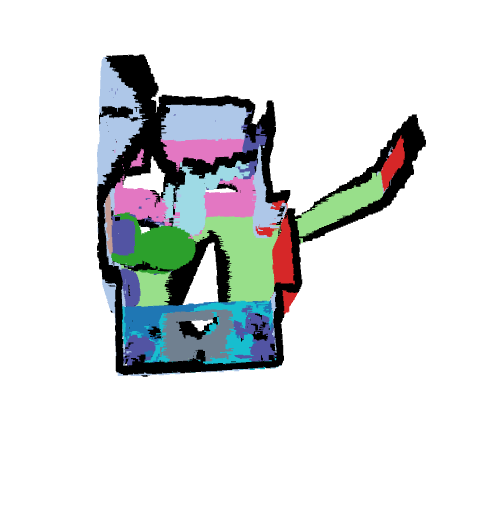}
\end{subfigure}

\caption{
Qualitative comparison between ground-truth semantic maps and Casper3D predictions.
Top: ground truth. Bottom: predicted outputs.
}
\label{fig:gt_pred_qualitative_1}
\end{figure}

\begin{figure}[t]
\centering

\begin{subfigure}[t]{0.48\linewidth}
    \centering
    \includegraphics[width=\linewidth]{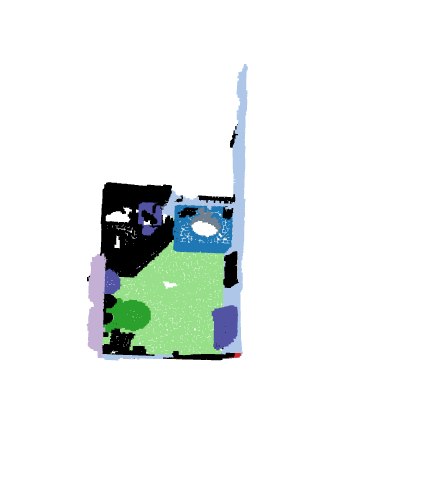}
\end{subfigure}
\hfill
\begin{subfigure}[t]{0.48\linewidth}
    \centering
    \includegraphics[width=\linewidth]{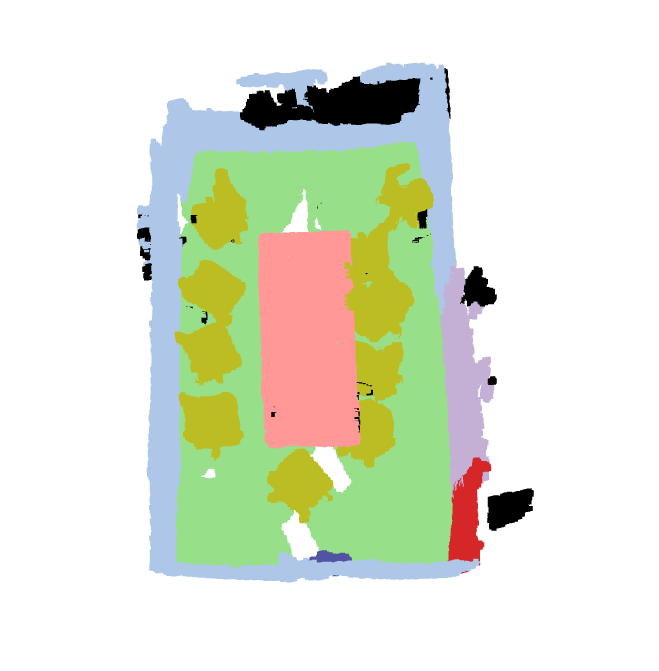}
\end{subfigure}

\vspace{4pt}

\begin{subfigure}[t]{0.48\linewidth}
    \centering
    \includegraphics[width=\linewidth]{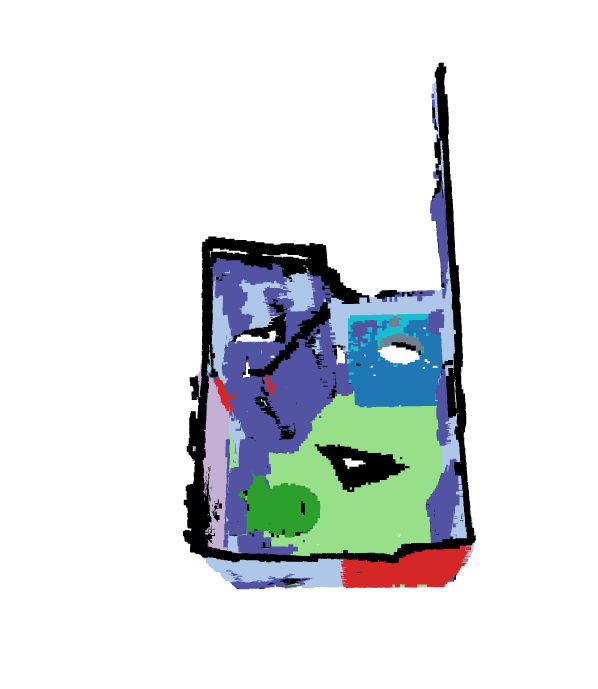}
\end{subfigure}
\hfill
\begin{subfigure}[t]{0.48\linewidth}
    \centering
    \includegraphics[width=\linewidth]{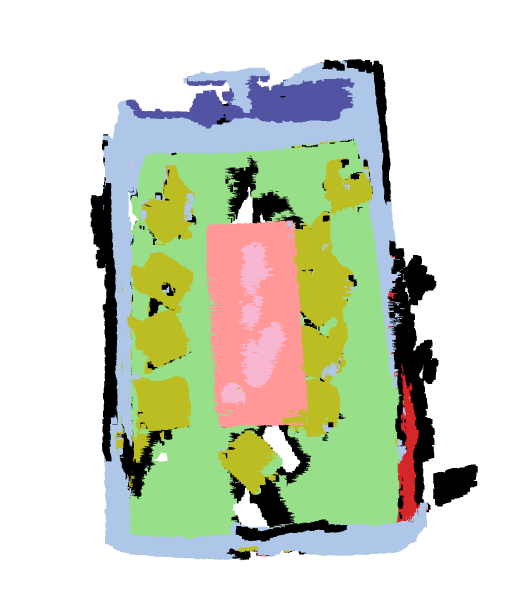}
\end{subfigure}
\caption{
Additional qualitative comparison between ground-truth semantic maps and Casper3D predictions.
Top: ground truth. Bottom: predicted outputs.
}
\label{fig:gt_pred_qualitative_2}
\end{figure}

\begin{figure}[t]
    \centering

    \begin{subfigure}{0.48\textwidth}
        \centering
        \begin{subfigure}{0.48\linewidth}
            \centering
            \includegraphics[width=\linewidth]{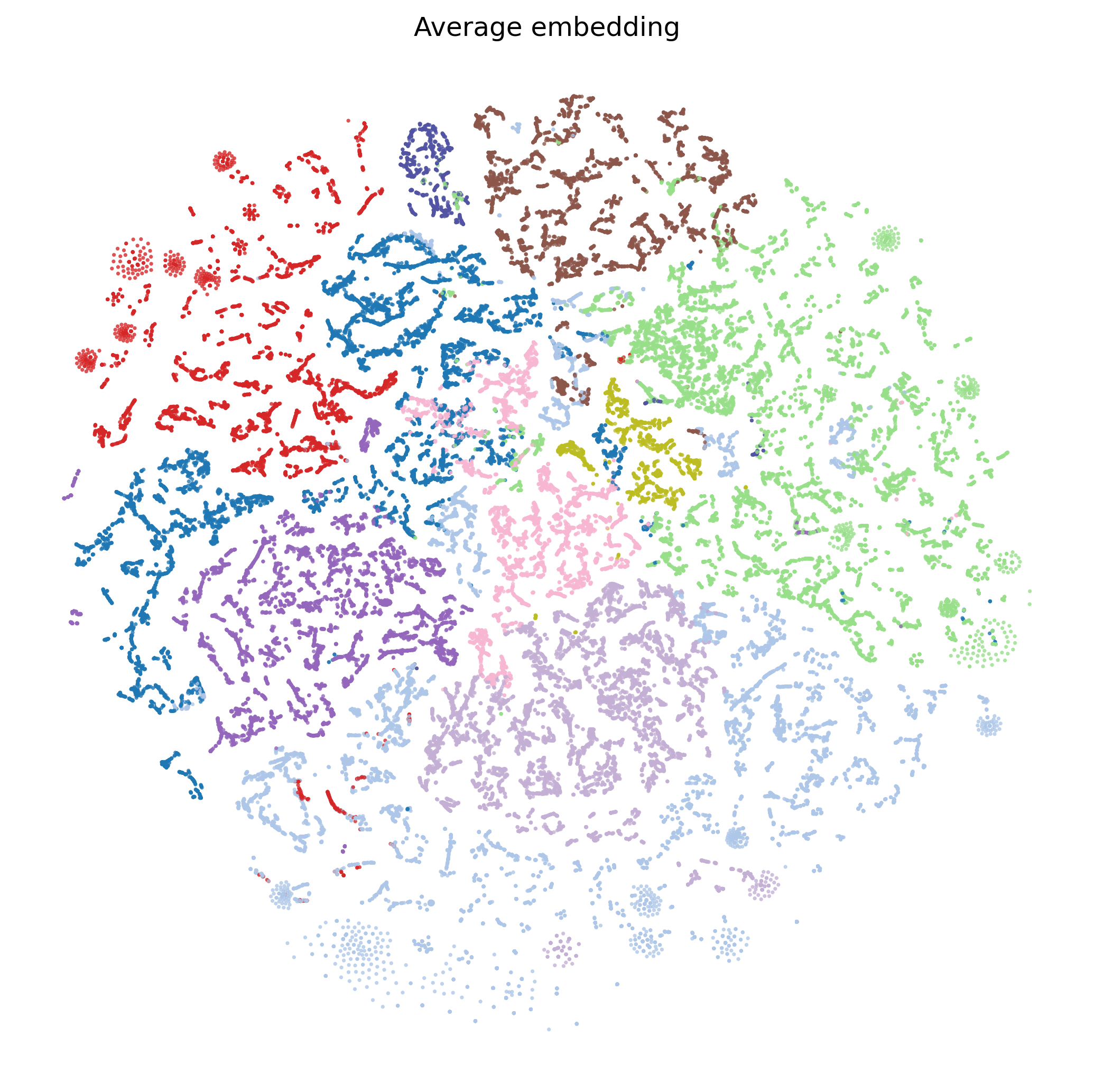}
            \caption*{Average}
        \end{subfigure}
        \hfill
        \begin{subfigure}{0.48\linewidth}
            \centering
            \includegraphics[width=\linewidth]{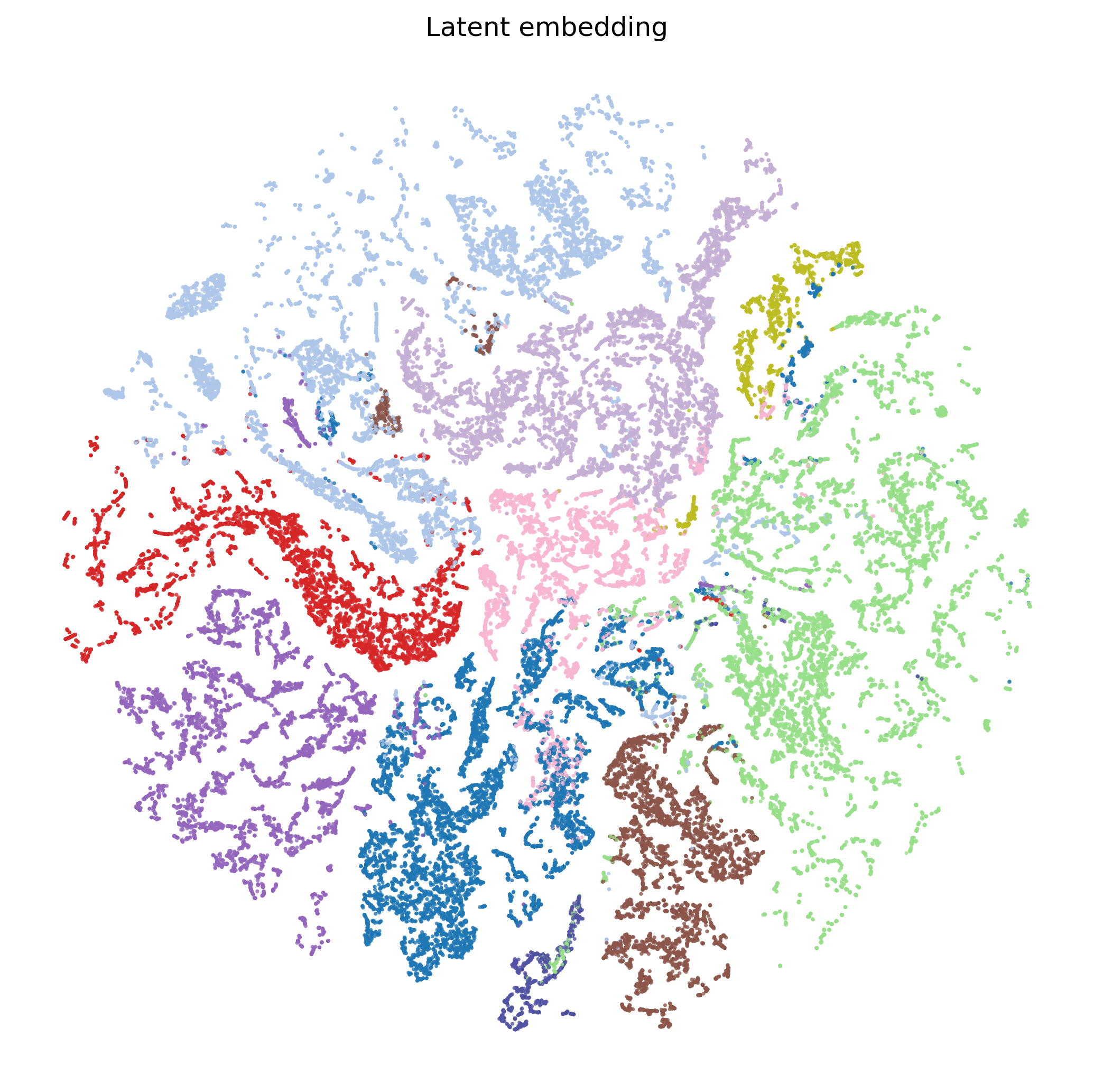}
            \caption*{Latent}
        \end{subfigure}
        \caption{Scene 1}
        \label{fig:tsne_pair4}
    \end{subfigure}
    \hfill
    \begin{subfigure}{0.48\textwidth}
        \centering
        \begin{subfigure}{0.48\linewidth}
            \centering
            \includegraphics[width=\linewidth]{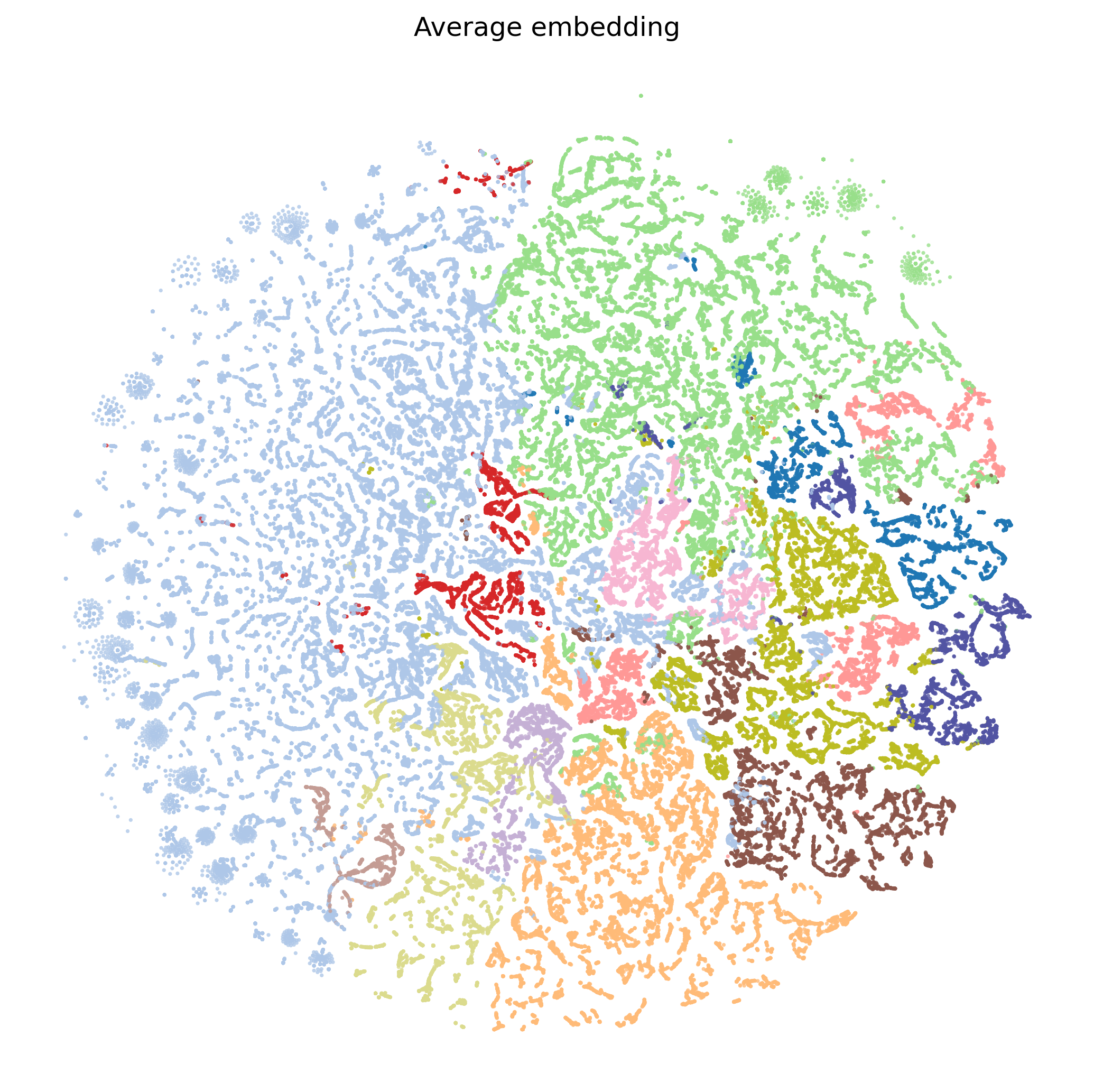}
            \caption*{Average}
        \end{subfigure}
        \hfill
        \begin{subfigure}{0.48\linewidth}
            \centering
            \includegraphics[width=\linewidth]{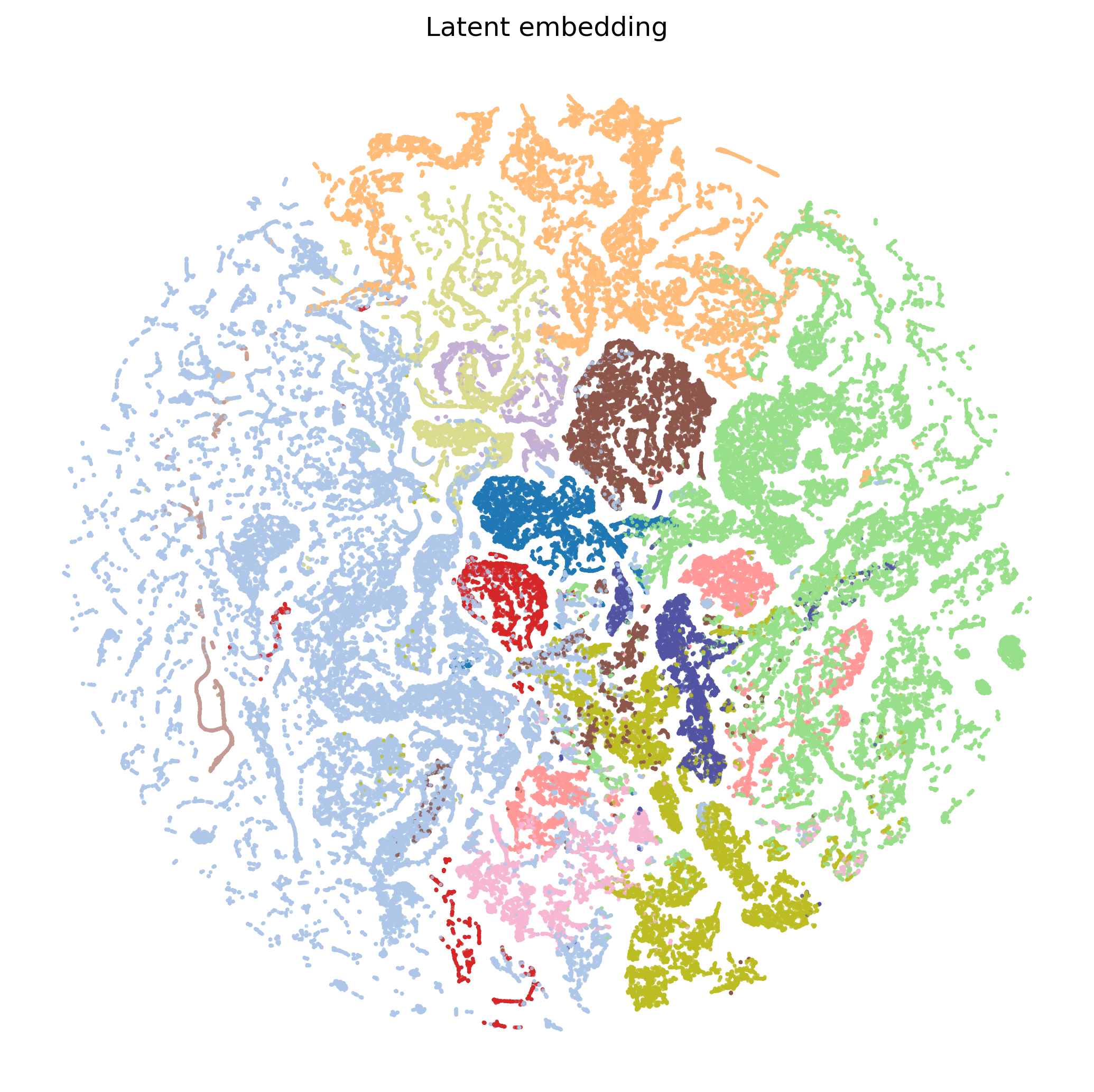}
            \caption*{Latent}
        \end{subfigure}
        \caption{Scene 2}
        \label{fig:tsne_pair6}
    \end{subfigure}

    \vspace{0.5em}

    \begin{subfigure}{0.48\textwidth}
        \centering
        \begin{subfigure}{0.48\linewidth}
            \centering
            \includegraphics[width=\linewidth]{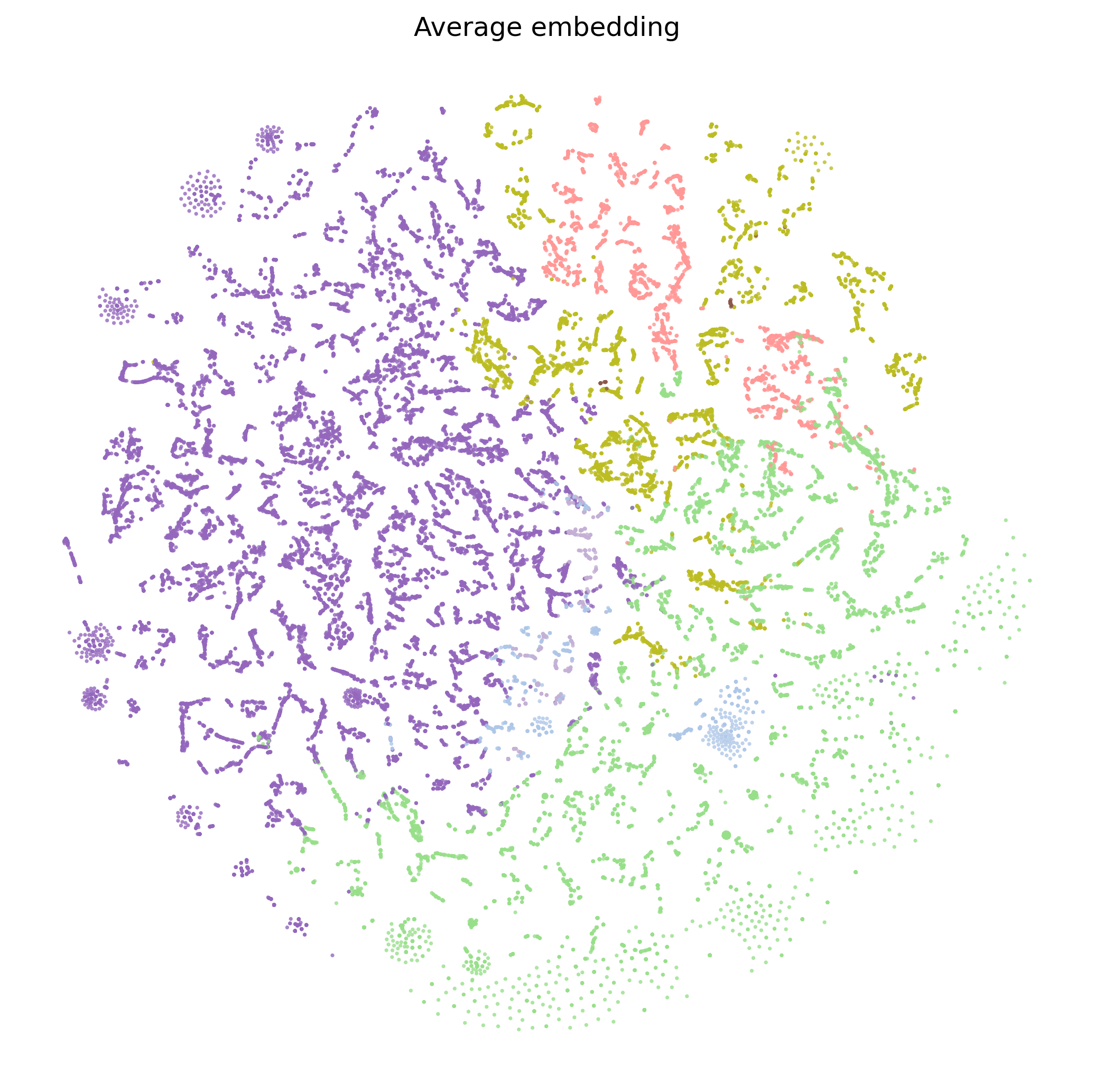}
            \caption*{Average}
        \end{subfigure}
        \hfill
        \begin{subfigure}{0.48\linewidth}
            \centering
            \includegraphics[width=\linewidth]{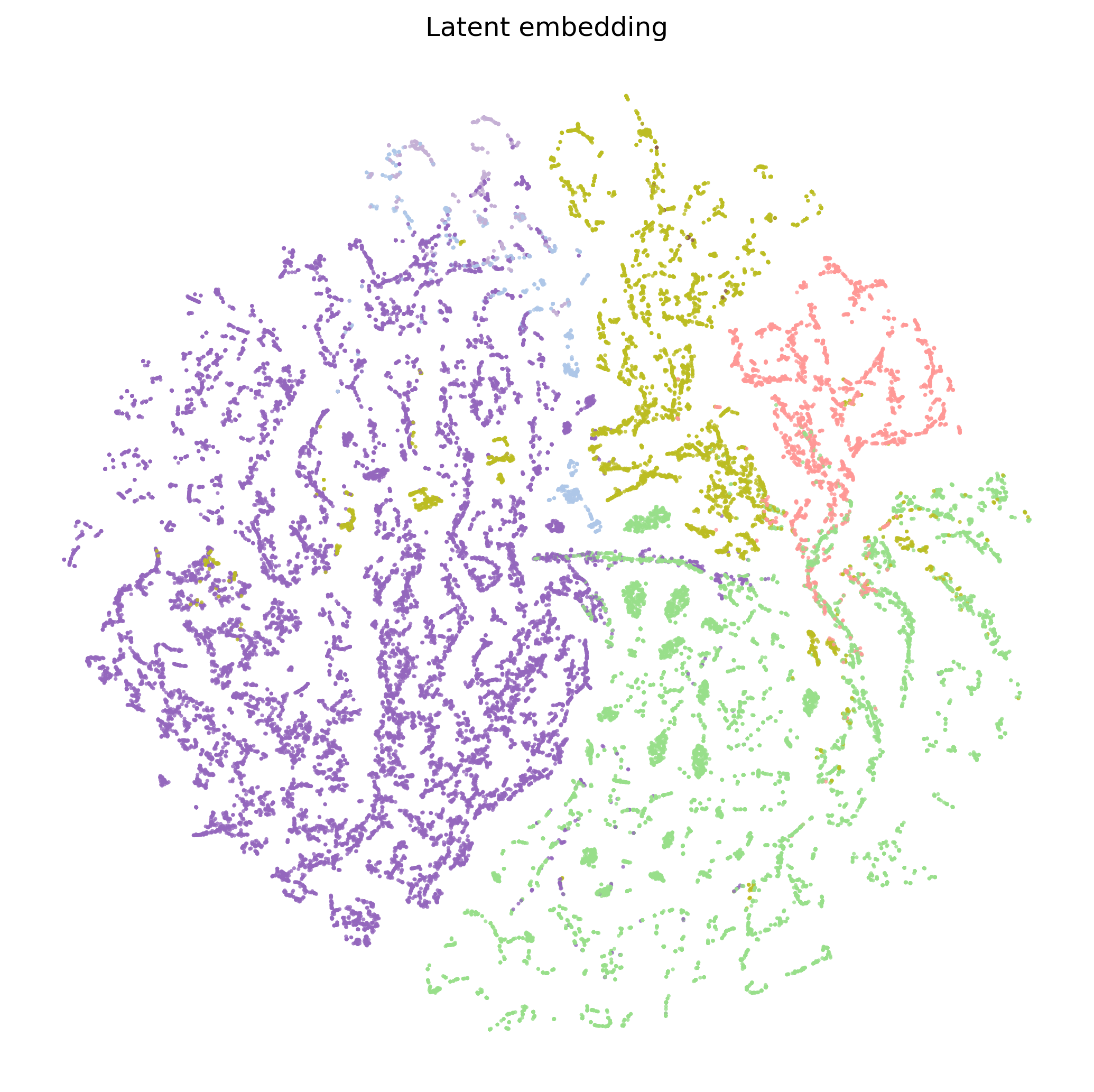}
            \caption*{Latent}
        \end{subfigure}
        \caption{Scene 3}
        \label{fig:tsne_pair1}
    \end{subfigure}
    \hfill
    \begin{subfigure}{0.48\textwidth}
        \centering
        \begin{subfigure}{0.48\linewidth}
            \centering
            \includegraphics[width=\linewidth]{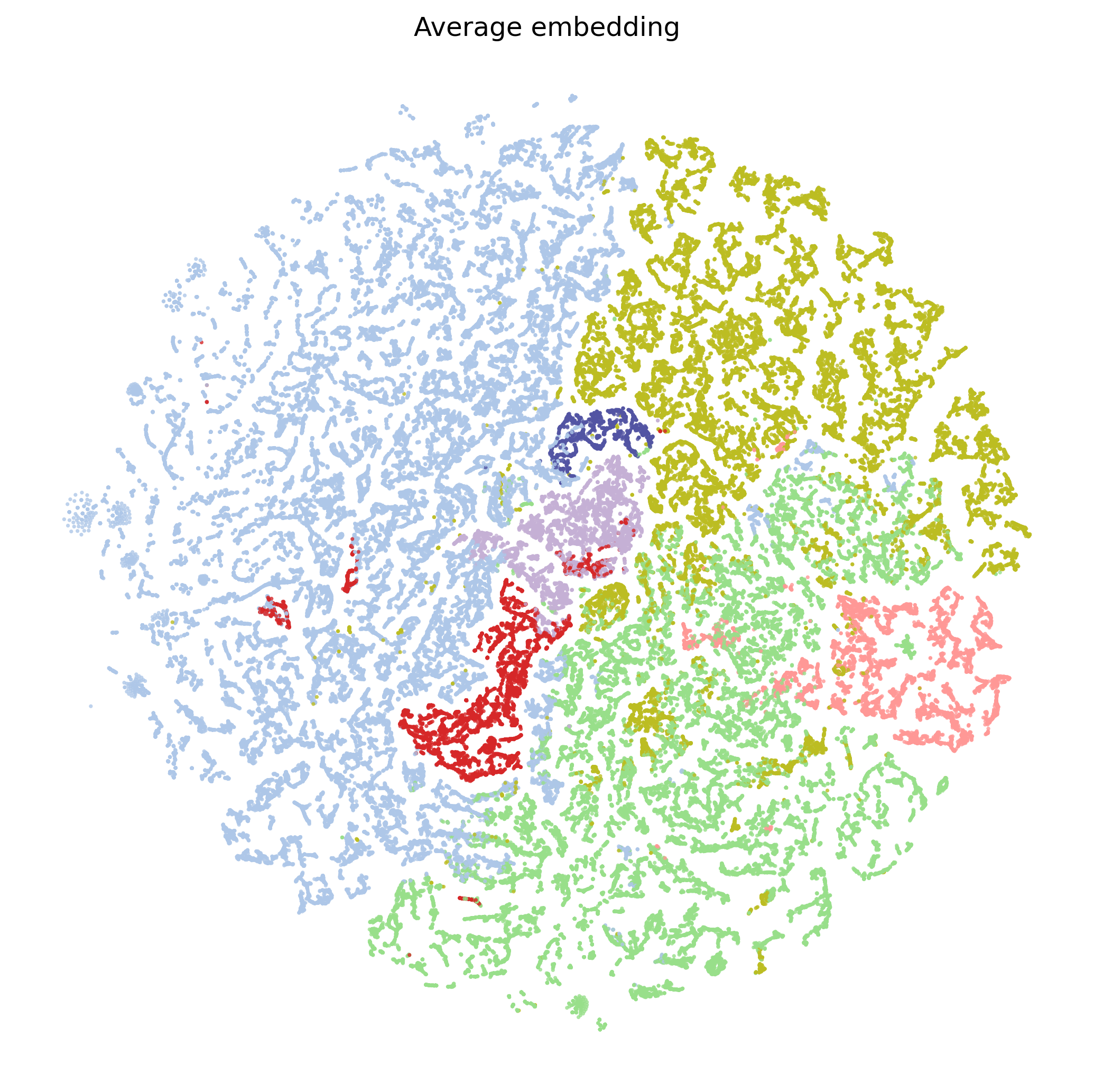}
            \caption*{Average}
        \end{subfigure}
        \hfill
        \begin{subfigure}{0.48\linewidth}
            \centering
            \includegraphics[width=\linewidth]{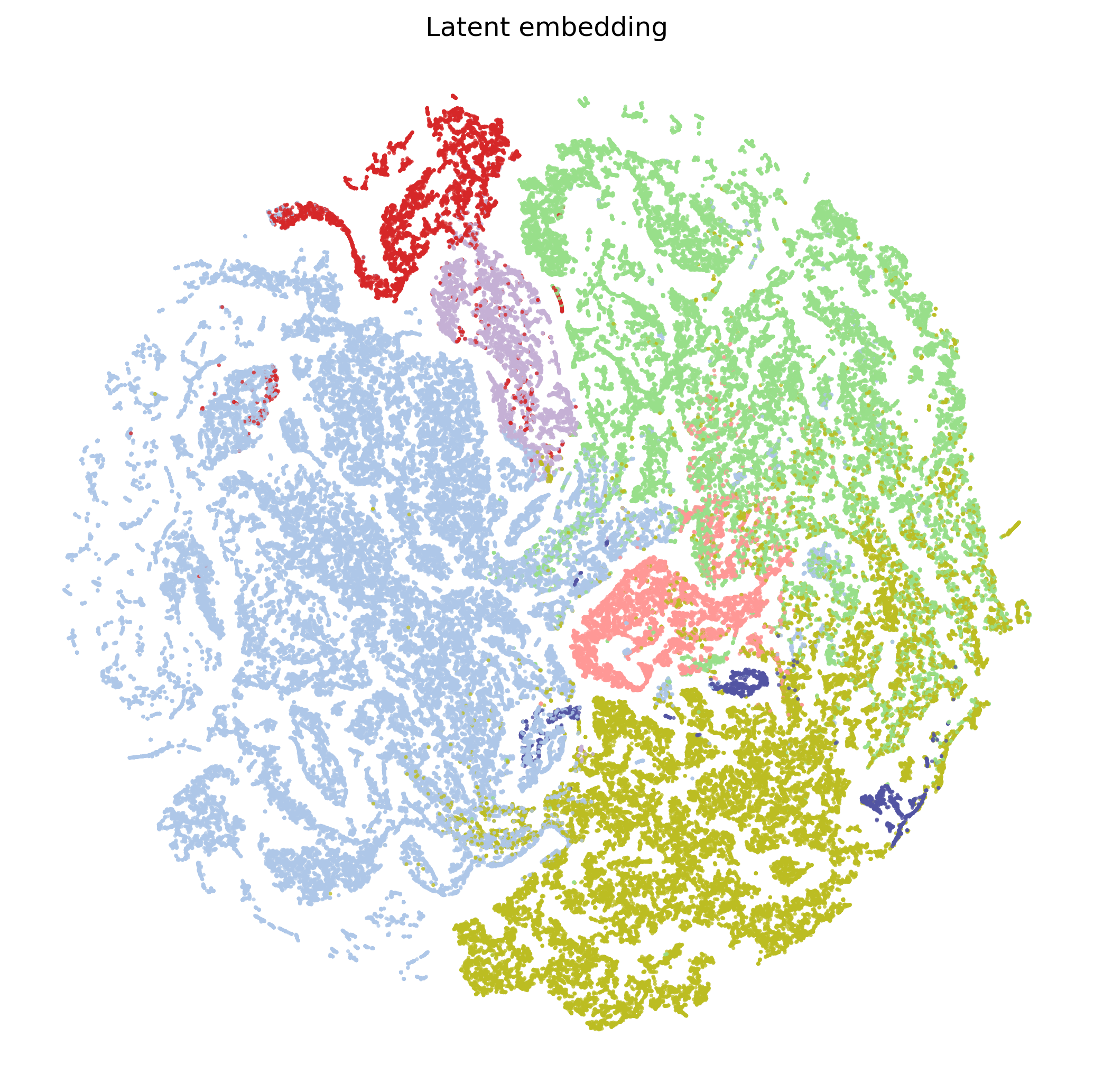}
            \caption*{Latent}
        \end{subfigure}
        \caption{Scene 4}
        \label{fig:tsne_pair8}
    \end{subfigure}

    \caption{
    t-SNE visualizations of point embeddings on representative ScanNet scenes. 
    Compared with deterministic average aggregation, the latent features inferred by \MyTitle\ form tighter same-class clusters and exhibit less inter-class mixing, indicating a more structured semantic field.
    }
    \label{fig:tsne_avg_vs_latent_2x2}
\end{figure}

\newpage

\end{document}